%% file: main.tex
\documentclass[letterpaper]{article} 
\usepackage{aaai25}  
\usepackage{times}  
\usepackage{helvet}  
\usepackage{courier}  
\usepackage[hyphens]{url}  
\usepackage{graphicx} 
\usepackage{natbib}  
\usepackage{caption} 
\usepackage{dsfont}
\usepackage{cite}
\usepackage{amsmath,amssymb,amsfonts}
\usepackage{graphicx}
\usepackage{textcomp}
\usepackage{xcolor}
\usepackage{algorithm}
\usepackage[noend]{algorithmic}
\usepackage{amsthm}
\usepackage{appendix}
\usepackage{url}
\usepackage{diagbox}
\usepackage{multirow} 
\usepackage{hhline}
\usepackage{comment}
\makeatletter
\newtheorem*{rep@theorem}{\rep@title}
\newcommand{\newreptheorem}[2]{%
\newenvironment{rep#1}[1]{%
 \def\rep@title{#2 \ref{##1}}%
 \begin{rep@theorem}}%
 {\end{rep@theorem}}}
\makeatother

\newtheorem{theorem}{Theorem}
\newreptheorem{theorem}{Theorem}
\newtheorem{lemma}[theorem]{Lemma}

\newtheorem{prop}{Proposition}
\usepackage{caption}
\usepackage{subcaption}
\usepackage[
    textsize=scriptsize,
    textwidth=20mm
]{todonotes} 

\usepackage{newfloat}
\usepackage{listings}
\DeclareCaptionStyle{ruled}{labelfont=normalfont,labelsep=colon,strut=off} 
\floatstyle{ruled}
\newfloat{listing}{tb}{lst}{}
\floatname{listing}{Listing}
\def\real{\mathbb{R}}

\def\E{\mathbb{E}}
\def\localupdaterule{\mathrm{Local\_Update}}
\def\dataset{{\mathcal{D}}}
\def\localdataset{\dataset_c}

\def\clients{{\mathcal{C}}}
\def\model{\theta}

\def\messages{{\mathcal{M}}} 
\def\roundsclient{{\mathcal{T}_c}} 

\def\globalmodel{\theta^*}
\def\localmodel{\theta_c^*}

\def\globallossfunction{{\mathcal{L}}}
\def\locallossfunction{{\mathcal{L}_c}}

\def\gradient{\mathbf{g}}
\def\cc{\mathbf{c}}
\def\locallossfunctioncc{{\mathcal{L}_\cc}}

\def\localmodelcc{\theta_\cc^*}
\def\x{\mathbf{x}_c}
\def\y{\mathbf{y}_c}

\def\h{\mathbf{H}}
\def\i{\mathbf{I}}
\def\w{\mathbf{W}}
\def\v{\mathbf{V}}

\def\a{\mathbf{A}}
\def\u{\mathbf{U}}

\def\e{\mathbf{\epsilon}}
\def\ga{\mathbf{\gamma}}
\def\var{\mathbb{V}\mathrm{ar}}

\def\s{\mathbf{s}_c}

\def\P{\mathbf{P}}
\def\ei{\mathbf{e}(\model)}
\def\e{\mathbf{e}}

\DeclareMathOperator*{\argmax}{argmax}
\DeclareMathOperator*{\argmin}{argmin}

\def\BibTeX{{\rm B\kern-.05em{\sc i\kern-.025em b}\kern-.08em
    T\kern-.1667em\lower.7ex\hbox{E}\kern-.125emX}}

\title{Attribute Inference Attacks for Federated Regression Tasks}
\author {
    Francesco Diana\textsuperscript{\rm 1, \rm 2},
    Othmane Marfoq\textsuperscript{\rm 3},
    Chuan Xu\textsuperscript{\rm 1, \rm 2, \rm 4, \rm 5},
    Giovanni Neglia\textsuperscript{\rm 1, \rm 2},
    Frédéric Giroire\textsuperscript{\rm 1, \rm 2, \rm 4, \rm 5},
    Eoin Thomas\textsuperscript{\rm 6}
}
\affiliations {
    \textsuperscript{\rm 1}Université Côte d’Azur\\
    \textsuperscript{\rm 2}Inria\\
    \textsuperscript{\rm 3}Meta\\
    \textsuperscript{\rm 4}CNRS\\
    \textsuperscript{\rm 5}I3S\\
    \textsuperscript{\rm 6}Amadeus\\
    \{francesco.diana, chuan.xu, giovanni.neglia, frederic.giroire\}@inria.fr, omarfoq@meta.com, eoin.thomas@amadeus.com
}

\begin{document}

\maketitle


\begin{abstract}
    Federated Learning (FL) enables multiple clients, such as mobile phones and IoT devices, to collaboratively train a global machine learning model while keeping their data localized. However, recent studies have revealed that the training phase of FL is vulnerable to reconstruction attacks, such as attribute inference attacks (AIA), where adversaries exploit  exchanged messages and auxiliary public information to uncover sensitive attributes of targeted clients. While these attacks have been extensively studied in the context of classification tasks, their impact on regression tasks remains largely unexplored. In this paper, we address this gap by proposing novel model-based AIAs specifically designed for regression tasks in FL environments. Our approach considers scenarios where adversaries can either eavesdrop on exchanged messages or directly interfere with the training process. We benchmark our proposed attacks against state-of-the-art methods using real-world datasets. The results demonstrate a significant increase in reconstruction accuracy, particularly in heterogeneous client datasets, a common scenario in FL. The efficacy of our model-based AIAs makes them better candidates for empirically quantifying privacy leakage for federated regression tasks.
\end{abstract}

\begin{links}
\link{Code}{https://github.com/francescodiana99/fedkit-learn}
\end{links}


\section{Introduction}

Federated learning (FL) enables multiple clients to collaboratively train a global model~\cite{mcmahan2017communication, lian17, li2018federated}. 
Since clients' data is not collected by a third party, FL naturally offers a certain level of privacy. Nevertheless, FL alone does not provide formal privacy guarantees, and recent works have demonstrated that clients' private information can still be easily leaked~\cite{surveyAttacks, liu2022threats}.
For instance, an adversary with access to the exchanged messages and knowledge of some public information (e.g., client's provided ratings)~\cite{attribute_fl_workshop, attribute_fl, feng2021attribute} can reconstruct a client's sensitive attributes (e.g., gender/religion) in an attack known as attribute inference attack (AIA).
Additionally, the adversary can reconstruct client's training samples such as images~\cite{geiping2020inverting, yin2021see}.

\begin{figure}[t]
    \centering
    \includegraphics[scale=0.2]{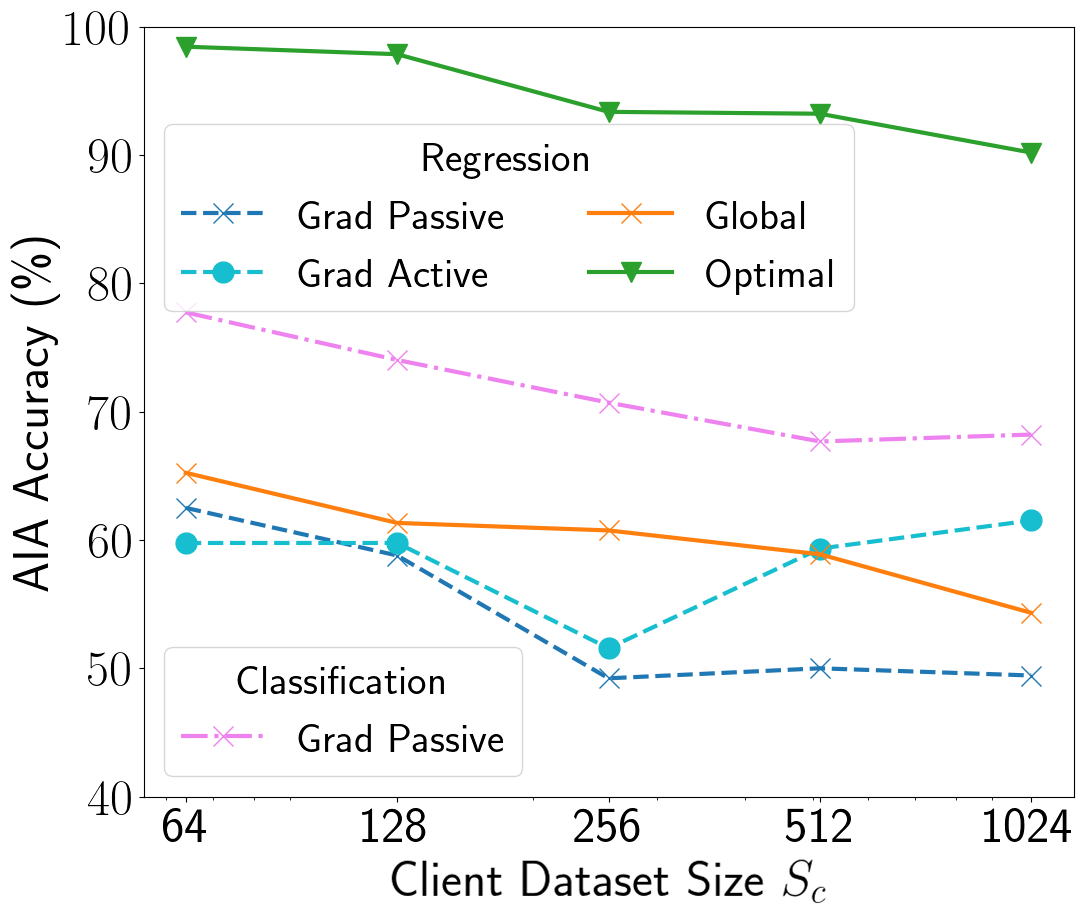}
    \caption{Average performance of different AIAs
    when four clients train a neural network through FedAvg with 1 local epoch and batch size 32. Each client stores $S_c$ data points randomly sampled from ACS Income dataset~\cite{income_dataset}. The adversary infers the gender attribute of every data sample held by the client given access to the released (public) information.}
    \label{fig:gradient_vs_model}
\end{figure}

However, these reconstruction attacks for FL have primarily been tested on classification tasks and have not been explored for \textit{regression} tasks, which are, needless to say, equally important for practical applications. 
Quite surprisingly, our experiments, as shown in Fig.~\ref{fig:gradient_vs_model}, demonstrate that the accuracy of state-of-the-art gradient-based AIA under a honest-but-curious adversary~\cite{attribute_fl_workshop, attribute_fl} (referred to as passive) drops significantly from $71\%$ on a classification task to $50\%$ (random guess) on a regression task once the targeted client holds more than 256 data points.
Furthermore,  even a more powerful (active) adversary capable of forging the messages to the targeted client to extract more information~\cite{attribute_fl_workshop, attribute_fl} offers only limited improvement to the AIA performance on regression tasks.
Detailed information about this experiment is in Appendix~\ref{app:figure1}. 

In this paper, we show that federated training of regression tasks does not inherently enjoy higher privacy, but it is simply more vulnerable to other forms of attacks. While existing FL AIA attacks for classification tasks are gradient-based (see Sec.~\ref{subsec:gradient-based}), we show that model-based AIAs---initially proposed for centralized training~\cite{fredrikson2014privacy, kasiviswanathan2013power, yeom2018privacy}---may be more effective for regression tasks.  Figure~\ref{fig:gradient_vs_model} illustrates that a  model-based attack on the server's global model (i.e., the final model trained through FL) already performs at least as well as the SOTA gradient-based passive attack. Moreover, it highlights that even more powerful attacks (up to 30~p.p.~more accurate) could be launched if the adversary had access to the optimal local model of the targeted client (i.e., a model trained only on the client's dataset).

Motivated by these observations, we propose a new two-step model-based AIA for \emph{federated regression tasks}. In this attack, the adversary first (approximately) reconstructs the client's optimal local model and then applies an existing model-based AIA to that model.

Our main contributions can be summarized as follows:
\begin{itemize}
    \item  We provide an analytical lower bound for model-based AIA accuracy in the least squares regression problem. This result motivates the adversary's strategy to approximate the client's optimal local model in federated regression tasks (Sec.~\ref{sec:model-based-guarantees}).
    \item   We propose methods for approximating optimal local models where adversaries can either eavesdrop on exchanged messages or directly interfere with the training process (Sec.~\ref{sec:lmra}).  
    \item Our experiments show that our model-based AIAs are better candidates for empirically quantifying privacy leakage for federated regression tasks  (Sec.~\ref{sec:exp}).
\end{itemize}


\section{Preliminaries}\label{sec:prelimi}
\subsection{Federated Learning}
We denote by $\clients$ the set of all clients participating to FL. 
Let $\localdataset = \{(\x(i), y_c(i)), i=1,..., S_c\}$ denote the local dataset of client $c\in \clients$ with size $S_c\triangleq |\localdataset|$.  Each data sample $(\x(i), y_c(i))$ is a pair consisting of an input $\x(i)$ and of an associated target value $y_c(i)$.
In FL, clients cooperate to learn a global model, which minimizes the following empirical risk over all the data owned by clients:
\begin{align}
\label{eq:globalobjective}
\min_{\model \in \real^d} \globallossfunction(\model) 
&= \sum_{c \in \clients} p_c \locallossfunction(\model, \dataset_c) \notag \\
&= \sum_{c \in \clients} p_c \left( \frac{1}{S_c} \sum_{i=1}^{S_c} \ell(\model, \x(i), y_c(i)) \right).
\end{align}
where $\ell(\model, \x(i), y_c(i))$ measures the loss of the model $\model$ on the sample $(\x(i), y_c(i))\in \localdataset$ and $p_c$ is the positive weight of client $c$, s.t.,~$\sum_{c\in \clients} p_c = 1$. Common choices of weights are $p_c =\frac{1}{|\clients|}$ or $p_c =\frac{S_c}{\sum_{c\in \clients}S_c}$.

\begin{algorithm}[t]
\caption{FL Framework}\label{algo:fl}
\textbf{Output}: $\model^T$
\begin{algorithmic}[1]

 \item[Server $s$:]
    \FOR{$t\in \{0,...,T-1\}$}
        \STATE{$s$ selects a subset of the clients $\clients^t\subseteq \clients$, }\label{algoline:select_client}
        \STATE{$s$ broadcasts the global model $\model^t$ to 
        $\clients^t$,} \label{algoline:broadcast}
        \STATE{$s$ waits for the updated models $\model_c^t$ 
        from every client $c\in \clients^t$,}
        \STATE{$s$ computes $\model^{t+1}$ by aggregating the received updated models.}
        \label{algoline:aggregation}
    \ENDFOR
 \item[Client $c\in \clients$: Input $\model$, Output $\model_c$]
 \WHILE{FL training is not completed}
    \STATE{$c$ listens for the arrival of new global model $\model$,}
   \STATE{$c$ updates its local model:
    $\model_c \leftarrow \localupdaterule^c (\model, \dataset_c)$\label{algoline:updaterule}}
    \STATE{$c$ sends back $\model_c$ to the server.}
 \ENDWHILE
\end{algorithmic}
\end{algorithm}

Let $\globalmodel= \arg\min_{\model \in \real^d} \globallossfunction(\model)$ be a global optimal model, i.e., a minimizer of Problem~\eqref{eq:globalobjective}. 
A general framework to learn such a global model in a federated way is shown in Alg.~\ref{algo:fl}; it generalizes a large number of FL algorithms, including FedAvg~\cite{mcmahan2017communication}, FedProx~\cite{li2018federated}, and FL with different client sampling techniques~\cite{nishio2019client, chen2020optimal, cho2020client}. The model~$\model^T$---the output of Alg.~\ref{algo:fl}---is the tentative solution of Problem~\eqref{eq:globalobjective}. Its performance depends on the specific FL algorithm, which precises how clients are selected in line~\ref{algoline:select_client}, how the updated local models are aggregated in line~\ref{algoline:aggregation}, and how the local update rule works in line~\ref{algoline:updaterule}.
For example, in FedAvg, clients are selected uniformly at random among the available clients, the local models are averaged with constant weights, and the clients perform locally multiple stochastic gradient steps \cite{mcmahan2017communication}.

\subsection{Threat Model}\label{subsection:threadmodel}

\subsubsection{Honest-but-Curious Adversary.}
We describe first an honest-but-curious adversary,\footnote{In what follows, we refer to the client using female pronouns and the adversary using male pronouns, respectively.} which is a standard threat model in existing literature~\cite{melis2019exploiting, geiping2020inverting, yin2021see, nasr2019comprehensive}, including the FL one~\cite[Table~7]{openproblems}, 
\cite{attribute_fl_workshop, attribute_fl}. This \emph{passive} adversary, who could be the server itself, is 
knowledgeable about the trained model structure, the loss function, and the training algorithm, and may eavesdrop on communication between the attacked client and the server but does not interfere with the training process. For instance, during training round $t$, the adversary can inspect the messages exchanged between the server and the attacked client (denoted by $c$), allowing him to recover the parameters of the global model $\model^t$ and the updated client model $\model_c^t(K)$. 
Let $\roundsclient \subseteq \{t|c\in\clients^t, \forall t\in {0,\dots,T-1}\}$ denote the set of communication rounds during which the adversary inspects messages exchanged between the server and the attacked client and $\messages_c = \{(\model^t, \model_c^t), \forall t\in \roundsclient\}$ denote the corresponding set of messages.

When it comes to defenses against such an adversary, it is crucial to understand that traditional measures like encrypted communications are ineffective if the attacker is the FL server.
More sophisticated cryptographic techniques like secure aggregation protocols~\cite{bonawitz_secureagg, kadhe2020fastsecagg} allow the server to aggregate local updates without having access to each individual update and, then, do hide the client's updated model $\model_c^t$ from the server. Nevertheless, they come with a significant computation overhead~\cite{quoc2020securetf} and are inefficient for sparse vector aggregation~\cite{openproblems}. Moreover, they are vulnerable to poisoning attacks, as they hinder the server from detecting (and removing) potentially harmful updates from malicious clients~\cite{blanchard2017machine, yin2018byzantine,elmhamdiRobustDistributedLearning2020}. For instance, \citet[Sec.~4.4]{truthSerum} introduce a new class of data poisoning attacks that succeed when training models with secure multiparty computation. Alternatively, Trusted Execution Environments (TEEs) ~\cite{TEE, tee2} provide an encrypted memory region to ensure the code has been executed faithfully and privately. They can then both conceal clients' updated models and defend against poisoning attacks. However, implementing a reliable TEE platform in FL remains an open challenge due to the infrastructure resource constraints and the required communication processes needed to connect verified codes~\cite{openproblems}.

\subsubsection{Malicious Adversary.} 
We also consider a stronger \emph{active} adversary who can interfere with the training process. Specifically, this adversary can modify the messages sent to the clients and have the clients update models~$\model^t$ that have been concocted to reveal more private information. 
 Let $\mathcal{T}^a_c $ be the set of rounds during which the adversary attacks client $c$ by sending a malicious model $\model^t$. 
As above, the adversary could be the server itself. This adversary has been widely considered in the literature on reconstruction attacks~\cite{fishing2022, active_sra} and membership inference attacks~\cite{active_membership, nasr2019comprehensive}. Some studies have also explored the possibility of a malicious adversary modifying the model architecture during training~\cite{robbing, zhao2023loki}, even though such attacks appear to be easily detectable. In this paper, we do not allow the adversary to modify the model architecture. 
For simplicity, we will refer to these two adversaries as passive and active adversaries, respectively, throughout the rest of the paper.

\subsection{Attribute Inference Attack (AIA) for FL}

AIA leverages public information to deduce private or sensitive attributes~\cite{kasiviswanathan2013power,fredrikson2014privacy,yeom2018privacy,attribute_fl_workshop, attribute_fl}.
For example, an AIA could reconstruct a user's gender from a recommender model by having access to the user's provided ratings.
Formally, each input $\x(i)$ of client $c$ consists of public attributes $\x^p(i)$ and of a sensitive attribute~$s_c(i)$. The target value, assumed to be public, is denoted by $y_c^p(i)$. The adversary, having access to $\{(\x^p(i), y_c^p(i)), i=1,..., S_c\}$ and $\messages_c$, aims to recover the sensitive attributes ${s_c(i)}$.\footnote{In ~\cite{fredrikson2014privacy, yeom2018privacy}, the adversary possesses additional information, including
estimates of the marginals or the joint distribution of the data samples. However, in this paper, we do not consider such a more powerful adversary.}

\subsubsection{Existing Gradient-Based AIA for FL.}\label{subsec:gradient-based} \citet{attribute_fl_workshop} and \citet{attribute_fl} present AIAs specifically designed for the FL context and both passive and active adversaries.
The central idea involves identifying sensitive attribute values that yield virtual gradients closely resembling the client's model updates---referred to as \emph{pseudo-gradients}---in terms of cosine similarity. Formally, the adversary solves the following optimization problem:
{\scriptsize \begin{equation}\label{eq:aia_gradient_based}
  \argmax_{\{s_c(i)\}_{i=1}^{S_c}} \sum_{t\in \mathcal{T}}\textbf{CosSim} \left(\frac{\partial\ell(\model^t, 
    \{(\x^p(i), s_c(i), y_c^p(i))\})}{\partial \model^t}, \model^t-\model_c^t\right), 
\end{equation}}
where $\mathcal{T}\subseteq \mathcal{T}_c$ for a passive adversary and  $\mathcal{T}\subseteq \mathcal{T}_c \cup \mathcal{T}_c^a$ for an active adversary. 
The active adversary simply sends back to the targeted client $c$ her own model $\model_c^{(t-1)}$ at each attack round in $\mathcal{T}_c^a$.

\citet{attribute_fl_workshop} and \citet{attribute_fl} assume that the sensitive attributes are categorical. 
Nevertheless, problem~\eqref{eq:aia_gradient_based} can be solved efficiently using a gradient method with the reparameterization trick and the Gumbel softmax distribution~\cite{jang2017categorical}. 
From~\eqref{eq:aia_gradient_based}, we observe that, since gradients incorporate information from all samples, the attack performance deteriorates in the presence of a large local dataset.
For example, the attack accuracy almost halves on the Genome dataset for the classification task when the client's local dataset size increases from 50 to 1000 samples~\cite[Table~8]{attribute_fl_workshop}. Our experiment (Figure~\ref{fig:gradient_vs_model}) on a regression task corroborates this finding: when the local dataset size increases from 64 to 256, the attack accuracy drops from $60\%$ to the level of random guessing.

\subsubsection{Model-Based AIA.} As an alternative, the AIA can be executed directly on the model (rather than on the model pseudo-gradients), as initially proposed for  centralized training in~\cite{kasiviswanathan2013power, fredrikson2014privacy}. Given a model $\model$, the adversary can infer the sensitive attributes by solving the following optimization problems:
\begin{equation}\label{eq:aia_model_based}
     \argmin_{s_c(i)} \ell(\model,(\x^p(i), s_c(i), y_c^p(i))),\,\,\, \forall i\in\{1,...,S_c\}, 
\end{equation}

Below we provide theoretical guarantees (Prop.~\ref{prop:attribute_nonprivate}) for the accuracy of model-based AIAs to least squares regression problems. 
Our theoretical result corroborates that, in an FL setting,  an adversary can benefit by first estimating the client’s optimal local model---that is, the model that minimizes the client’s empirical loss.


\section{Model-Based AIA Guarantees for Least Squares Regression}
\label{sec:model-based-guarantees}
In this section, we provide novel theoretical guarantees for the accuracy of the model-based AIA (Problem~\eqref{eq:aia_model_based}) in the context of least squares regression.  In particular, we show that the better the model $\model$ fits the local data and the more the sensitive attribute affects the final prediction, the higher the AIA accuracy.
 \begin{prop}\label{prop:attribute_nonprivate}
Let $E_c$ be the mean square error of a given least squares regression model $\model$ on the local dataset of client $c$ and $\model[s]$ be the model parameter corresponding to a binary sensitive attribute. The accuracy of the model-based AIA \eqref{eq:aia_model_based} is larger than or equal to $1-\frac{4E_c}{\model[s]^2}$. 
\end{prop}

\begin{proof}
Let $\s $ be the vector including all the unknown sensitive binary attributes $\{s_c(i), \forall i\in \{1, ..., S_c\}\}$ of client $c$.
Let $\x\in \real^{S_c\times d}$ be the design matrix with rank $d$ and $\y\in \real^{S_c}$ be the labels in the local dataset $\localdataset$ of the client $c$.  Let $\model[:p]\in \real^{d-1}$ be the parameters corresponding to the public attributes. 
The adversary has access to partial data instances in $\localdataset$ which consists of the public attributes $\P\in \real^{S_c\times(d-1)}$ and the corresponding labels $\y\in \real^{S_c}$. 

The goal for the adversary is to decode the values of the binary sensitive attribute $\s \in\{0,1\}^{S_c}$ given $(\P, \y)$ by solving~\eqref{eq:aia_model_based}, i.e., checking for each point, which value for the sensitive attribute leads to a smaller loss.

It is easy to verify that the problem can be equivalently solved through the following two-step procedure.
First, the adversary computes the vector of real values:
$$
\tilde{\mathbf{s}}_c = \argmin_{\s\in \real^{S_c}} ||\P\model[:p]+ \s\model[s] - \y ||_2^2 = \frac{\y-\P\model[:p]}{\model[s]}.  
$$
Then, the adversary reconstruct the vector of sensitive features $\hat{\mathbf{s}}_c \in \{0,1\}^{S_c}$ as follows
\begin{equation}\label{eq:sensitive_prediction_condition}
\hat s_c(i) = 
\begin{cases}
    0& \text{if } \tilde{{s}}_c(i) <\frac{1}{2}\\
    1              & \text{otherwise}
\end{cases}
,\,\,\,\forall i\in \{1, ..., S_c\}. 
\end{equation}

Let $\e$ be the vector of residuals for the local dataset, i.e., $\e_c = \y - (\P\model[:p]+ \s\model[s])$. We have then 
\begin{equation}\label{eq:sensitive_relation}
    \s = \frac{\y - \P \model[:p] - \e_c}{\model[s]} = \tilde{\mathbf{s}}_c - \frac{\e_c}{\model[s]}. 
\end{equation}

Let us say that the sensitive feature of sample $i$ has been erroneously reconstructed, i.e., $s_c(i) \neq \hat s_c(i)$, then  \eqref{eq:sensitive_prediction_condition}, implies that $|s_c(i) - \tilde s_c(i)|\ge 1/2$, and from
 \eqref{eq:sensitive_relation} it follows that $|e_c(i)|^2 \ge \model[s]^2/4$. As $E_c S_c= \lVert \e_c \rVert_2^2$, there can be at most $4 E_c S_c/\model[s]^2$ samples erroneously reconstructed, from which we can conclude the result.
 \end{proof}

Proposition~\ref{prop:attribute_nonprivate} indicates that the model-based AIA performs better the more the model overfits the dataset.
This observation justifies the following two-step attack in a FL setting: 
\emph{the adversary first reconstructs the optimal local model of the targeted client $c$, i.e., $\localmodel = \arg\min_{\model\in \real^d} \locallossfunction(\model, \localdataset)$, and then execute the AIA in \eqref{eq:aia_model_based} on the  reconstructed model.}\footnote{
Note that this model is optimal in terms of the training error, but not in general of the final test error. On the contrary, it is likely to overfit the local dataset.
}
In the following, we will detail our approaches for reconstructing the optimal local model by passive and active adversary, respectively.    


\section{Reconstructing the Local Model  in FL}\label{sec:lmra}
In this section, we show how an adversary may reconstruct the optimal local model of client $c$, i.e.,~$\localmodel = \arg\min_{\model\in \real^d} \locallossfunction(\model, \localdataset)$.

First, we provide an efficient approach for least squares regression under a passive adversary. We prove that the adversary can exactly reconstruct the optimal local model under deterministic FL updates and provide probabilistic guarantees on the reconstruction error under stochastic FL updates (Sec.~\ref{sec:exact}).

Second, we show that an active adversary can potentially reconstruct any client's local model (not just least square regression ones) in a federated setting (Sec.~\ref{subsec:active}).

\subsection{Passive Approach for Linear Least Squares}\label{sec:exact}

We consider that clients cooperatively train a linear regression model with quadratic loss. We refer to this setting as a federated least squares regression.  
The attack is detailed in Alg.~\ref{algo:lineardecode} and  involves a single matrix computation (line~\ref{algoline:ols}) after the exchanged messages $\messages_c$ have been collected. $\mathbf{A}^\dagger$ represents the pseudo-inverse of matrix $\mathbf{A}$.  
Theorem~\ref{prop:stochastic_case} provides theoretical guarantees  for the distance between the reconstructed model and the optimal local one, when the model is trained through FedAvg~\cite{mcmahan2017communication} with batch size $B$ and local epochs $E$. The formal statement of the theorem and its proof are in Appendix~\ref{app:stochasticproof}.

\begin{algorithm}[t]
\caption{Reconstruction of client-$c$ local model by a passive adversary for federated least squares regression }\label{algo:lineardecode}
\textbf{Input}: 
    the server models $\model^{t_i}(0)=\model_c^{t_i}(0)$ and the local updated models  $\model_c^{t_i}(K)$ at all the inspected rounds $t_i\in  \mathcal{T}_c= \{t_1, t_2, \dots, t_{n_c}\} $.
\begin{algorithmic}[1]
\small
    \STATE{Let $\Theta_{\text{in}}=[\model_c^{t_1}(0) \,\, \model_c^{t_2}(0) \,\, ... \,\, \model_c^{t_{n_c}}(0)]^T \in \real^{n_c\times{d}}$}
    \STATE{Let ${\Theta_{\text{out}}= \left[\begin{array}{cc}(\model_c^{t_1}(0)-\model_c^{t_1}(K))^{T} & 1 \\ \vdots & \\ (\model_c^{t_{n_c}}(0)-\model_c^{t_{n_c}}(K))^{T} & 1\end{array}\right] \in \real^{n_c\times(d+1)}}$}
    \STATE{
        ${(\hat\model_c^*)^T \leftarrow \text{last row of} \left( (\Theta_{\text{out}}^T \Theta_{\text{out}})^{\dagger}\Theta^T_{\text{out}}\Theta_{\text{in}} \right)}$ 

    }\label{algoline:ols}
    
    \STATE{Return $\hat \model_c^*$ as the estimator for client $c$'s local model}
    
\end{algorithmic}
\end{algorithm}

\begin{theorem}[Informal statement]\label{prop:stochastic_case}    
Consider a federated least squares regression with a large number of clients and assume that i)~client $c$ has $d$-rank design matrix $\x \in \real^{S_c\times d}$, ii)~she updates the global model through stochastic gradient steps  with sub-Gaussian noise with scale $\sigma$, iii)~the global model is independent of previous target client updates, and iv)~the eigenvalues of the matrix $\frac{\Theta_{\text{out}}^{T} \Theta_{\text{out}}}{n_c}$ are lower bounded by $\underline{\lambda}>0$. By eavesdropping on $n_c>d$ message exchanges between client $c$ and the server, the error of the reconstructed model $\hat\model^*_c$ of Alg.~\ref{algo:lineardecode} is upper bounded w.p.  $\ge 1-\delta$  when $\eta\leq \frac{S_c}{2\lambda_{\max}(\x^T\x)}$ and
\begin{equation}\label{eq:bounds}
  \left\lVert \hat\model^*_c - \localmodel \right\rVert_2 = \mathcal{O}\left(\eta\sigma d\sqrt{  dE\left\lceil\frac{S_c}{B}\right\rceil \frac{d+1+\ln \frac{2d}{\delta}}{n_c \cdot \underline{\lambda}}} \right). 
\end{equation}
\end{theorem}

When the batch size is equal to the local dataset size, Alg.~\ref{algo:lineardecode} exactly reconstructs the optimal local model (proof in Appendix\ref{app:proof_corollary}):

\begin{prop}\label{cor:linear_reconstruction}
Consider a federated least squares regression and assume that client $c$ has $d$-rank design matrix and updates the global model through full-batch gradient updates (i.e., $B=S_c$) and an arbitrary number of local epochs. Once the adversary eavesdrops on $d+1$ communication exchanges between client $c$ and the server, he can recover the client's optimal local model.
\end{prop}

Finally, the following proposition shows that 
our attack for the full-batch gradient case is \textbf{order-optimal} in terms of the number of messages the adversary needs to eavesdrop. The proof is in Appendix~\ref{app:proof_proposition}.

\begin{prop}\label{prop:order-optimal}
Consider that the federated least squares regression is trained through FedAvg with one local step ($E=1$) and full batch ($B=S_c$ for each client $c \in \mathcal{C}$). At least one client is required to communicate with the server $\Omega(d)$ times for the global model to be learned, and the adversary needs to eavesdrop at least $\Omega(d)$ messages from this client to reconstruct her optimal local model.
\end{prop}

\paragraph{Toy Example Illustration.} 
Here, we illustrate the performance of our Alg.~\ref{algo:lineardecode} on a toy dataset detailed in Appendix~\ref{app:toydataset}. Figure~\ref{fig:linear_toy} (left) demonstrates that 
as the batch size increases, the reconstructed model is closer to the optimal local model (as shown in our Theorem~\ref{prop:stochastic_case}). Moreover, since the model overfits more the dataset (with smaller loss), the AIA accuracy increases (Figure~\ref{fig:linear_toy} (right)). 
\begin{figure}[htbp]
\begin{subfigure}[t]{0.45\columnwidth}
    \centering
    \includegraphics[scale=0.21]{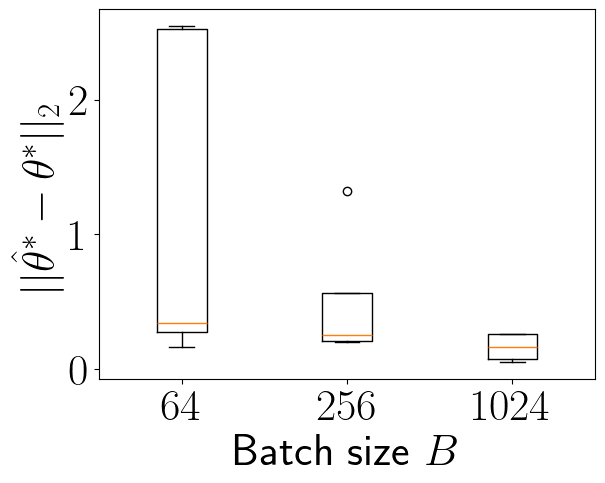}
\end{subfigure}
\begin{subfigure}[t]{0.45\columnwidth}
    \centering
    \includegraphics[scale=0.205]{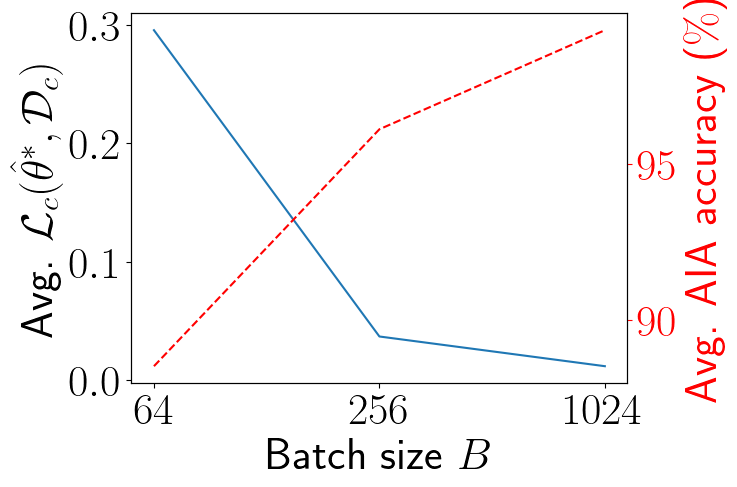}
\end{subfigure}
\caption{The performance of our passive approach for reconstructing optimal local model (left) and the triggered AIA (right) on a toy dataset with two clients training a linear model with size $d=11$ under batch size 64, 256 and 1024 for 5 seeds each, respectively. The passive adversary only eavesdropped $d+1$ messages.}
\label{fig:linear_toy}
\end{figure}

\subsection{Active Approach} \label{subsec:active}
Here we consider an active adversary (e.g., the server itself) that can modify the global model weights $\model^t$ to recover the client's optimal local model. 
To achieve this, the adversary can simply send back to the targeted client $c$ her own model $\model_c^{(t-1)}$ instead of the averaged model $\model^t$. In this way, client $c$ is led to compute her own optimal local model.

We propose a slightly more sophisticated version of this active attack. Specifically, we suggest that the adversary emulates the Adam optimization algorithm~\cite{adam} for the model updated by client $c$ by adjusting the learning rate for each parameter based on the magnitude and history of the gradients, and incorporating momentum. The motivation is twofold. First, the client does not receive back exactly the same model and thus cannot easily detect the attack. Second, Adam is known to minimize the training error faster than stochastic gradient descent at the cost of overfitting more to the training data~\cite{adam_generalize_1, adam_generalize}. We can then expect Adam to help the adversary reconstruct the client's optimal local model better for a given number of modified messages, and our experiments in Sec.~\ref{sec:exp}
 confirm that this is the case.

The details of this attack are outlined in Alg.~\ref{algo:active}. We observe that the adversary does not need to systematically modify all messages sent to the target client $c$ but can modify just a handful of messages that are not necessarily consecutive. This contributes to the difficulty of detecting the attack.

\begin{algorithm}[t]
\caption{Reconstruction of client-$c$ local model by an active adversary $a$}\label{algo:active}
 \textbf{Input}: Let $\mathcal{T}^a_c $ be set of rounds during which the adversary attacks client $c$ and $\model^a_c$ be the corresponding malicious model.
\begin{algorithmic}[1]
    \STATE{$\model_c^a \leftarrow$ latest model received from client $c$}
    \FOR{$t\in  \mathcal{T}^a_c$}
        \STATE{$a$ sends the model $\model^a_c$ to client $c$,}
        \STATE{$a$ waits the updated model from $\model_c$ 
        from client $c$,}
        \STATE{$a$ computes the pseudo-gradient $\model^a_c-\model_c$ and updates $\model^a_c$ and the corresponding moment vectors following Adam~\cite[Alg.~1]{adam},}
       
    \ENDFOR
\STATE{Return $\model^a_c$ as the estimator for client $c$'s local model}
\end{algorithmic}
\end{algorithm}


\section{Experiments}\label{sec:exp}
\subsection{Datasets} \label{subsec:datasets}
\paragraph{Medical~\cite{medical_dataset}.} This dataset includes 1,339 records and 6 features: age, gender, BMI, number of children, smoking status, and region. The \emph{regression} task is to predict each individual's medical charges billed by health insurance. The dataset is  split i.i.d. between 2 clients. 

\paragraph{Income~\cite{income_dataset}.} 
This dataset contains census information from 50 U.S. states and Puerto Rico, spanning from 2014 to 2018. It includes 15 features related to demographic information such as age, occupation, and education level. 
The \emph{regression} task is to predict an individual's income. 
We investigate two FL scenarios, named \mbox{Income-L} and Income-A, respectively. In Income-L, there are 10 clients holding only records from the  state of Louisiana (15,982 records in total). These clients can be viewed as the local entities working for the Louisiana State Census Data Center. We examine various levels of statistical heterogeneity among these local entities, with the splitting strategy detailed in Appendix~\ref{app:datasplit_incomel}. In Income-A, there are 51 clients, each representing a census region and collectively covering all regions. Every client randomly selects 20\% of the data points from the corresponding census region, resulting in a total of 332,900 records.

For all the datasets, each client keeps $90\%$ of its data for training and uses $10\%$ for validation.

\subsection{FL Training and Attack Setup}
In all the experiments, each client follows FedAvg to train a \emph{neural network} model with a single hidden layer of 128 neurons, using ReLU as activation function. 
The number of communication rounds is fixed to $T = \lceil 100/E \rceil$ where $E$ is the number of local epochs. Each client participates to all rounds, i.e.,~$\mathcal{T}_c = \{0,...,T-1\}$.
The learning rate is tuned for each training scenario (different datasets and number of local epochs), with values provided in Appendix~\ref{app:hyperparameters}. The passive adversary may eavesdrop all the exchanged messages until the end of the training. 
The active adversary launches the attack after $T$ rounds\footnote{We also examine the impact of the attack's starting round. The results are presented in Appendix~\ref{app:starting_points}} for additional $\lceil 10/E \rceil$ and $\lceil 50/E \rceil$ rounds.
Every attack is evaluated over FL trainings from 3 different random seeds.
For Medical dataset, the adversary infers whether a person smokes or not. For Income-L and Income-A datasets, the adversary infers the gender. 
The AIA accuracy is the fraction of the correct inferences over all the samples. 
We have also conducted experiments for federated least squares regression on the same datasets. The results can be found in Appendix\ref{app:linear_model}. 

\subsection{Baselines and Our Attack Implementation}\label{subsec:baseline}
\paragraph{Gradient-Based.}
We compare our method with the (gradient-based) SOTA (Sec.~\ref{subsec:gradient-based}). 
The baseline performance is affected by the set of  inspected rounds~$\mathcal{T}$ considered in~\eqref{eq:aia_gradient_based}. We select the inspected rounds $\mathcal{T}$ based on two criteria: the highest cosine similarity and the best AIA accuracy. In a real attack, the adversary is not expected to know the attack accuracy beforehand.
Therefore,  we refer to the attack based on the highest cosine similarity as {\bf{Grad}}  and to the other as {\bf{Grad-w-O}} ({\bf Grad}ient {\bf w}ith {\bf O}racle), as it assumes the existence of an oracle that provides the attack accuracy. The details for the tuning of $\mathcal{T}$ and other hyper-parameter settings can be found in Appendix~\ref{app:hyperparameters}.
   
\paragraph{Our Attacks.} Our attacks consist of two steps: 1) reconstructing the optimal local model, and 2) executing the model-based AIA in~\eqref{eq:aia_model_based}. For the first step, a passive adversary uses the last-returned model from the targeted client, while an active adversary  executes Alg.~\ref{algo:active}. The details of the hyperparameter settings can be found in Appendix~\ref{app:hyperparameters}.

\paragraph{Model-Based with Oracle (Model-w-O).} 
To provide an upper bound on the performance of our approach, we assume the existence of an oracle that provides the optimal local model for the first step of our attack. In practice, the optimal local model is determined empirically by running Adam for a very large number of iterations.

\subsection{Experimental Results}

\begin{table}[ht]
    \centering
    \begin{tabular}{|c|c|c|c|c|}
        \hline
         \multicolumn{2}{|c|}{\bfseries \backslashbox{AIA (\%)}{Datasets}}  & Income-L  &Income-A  &Medical \\
         \hhline{|=|=|=|=|=|}
    \multirow{3}{*}{\bfseries Passive}& Grad&60.36\scriptsize{$\pm$0.67} &54.98\scriptsize{$\pm$0.29} &87.26\scriptsize{$\pm$0.92} \\ 
    & Grad-w-O &71.44\scriptsize{$\pm$0.33}  &\bf{56.10}\scriptsize{$\pm$1.12} &91.06\scriptsize{$\pm$0.55} \\ 
    &Ours &\bf{75.27}\scriptsize{$\pm$0.32}  &55.75\scriptsize{$\pm$0.17}  & \bf{95.90}\scriptsize{$\pm$0.04}\\
    \hhline{|=|=|=|=|=|}
    \multirow{3}{*}{\shortstack[c]{\bfseries Active \\(10 Rnds)}}& Grad&60.24\scriptsize{$\pm$0.60}  &54.98\scriptsize{$\pm$0.29} &87.26\scriptsize{$\pm$0.92} \\ 
       & Grad-w-O &80.69\scriptsize{$\pm$0.55} &56.10\scriptsize{$\pm$1.12} &91.06\scriptsize{$\pm$0.55} \\ 
    & Ours&\bf{82.02}\scriptsize{$\pm$0.85} &\bf{63.53}\scriptsize{$\pm$0.73} &\bf{95.93}\scriptsize{$\pm$0.07} \\
    \hline
    \multirow{3}{*}{\shortstack[c]{\bfseries Active \\(50 Rnds)}}& Grad&60.24\scriptsize{$\pm$0.60}  & 53.36\scriptsize{$\pm$0.40} &87.26\scriptsize{$\pm$0.92} \\ 
     & Grad-w-O &80.69\scriptsize{$\pm$0.55}  &56.12\scriptsize{$\pm$0.12} &91.06\scriptsize{$\pm$0.55} \\ 
    & Ours&\bf{94.31}\scriptsize{$\pm$0.11} &\bf{78.09}\scriptsize{$\pm$0.25} &\bf{96.79}\scriptsize{$\pm$0.79} \\
    \hhline{|=|=|=|=|=|}
    \multicolumn{2}{|c|}{\bfseries Model-w-O}&94.31\scriptsize{$\pm$0.11} &78.31\scriptsize{$\pm$0.07} &96.79\scriptsize{$\pm$0.79} \\ 
    \hline 
    \end{tabular}
    \caption{The AIA accuracy over all clients' local datasets evaluated under both honest-but-curious (passive) and malicious (active) adversaries across  Income-L, Income-A, and Medical FL datasets~(Sec.~\ref{subsec:datasets}). 
    The standard deviation is evaluated over three FL training runs with different random seeds.
    All clients run FedAvg with 1 epoch and batch size 32. For Income-L, we consider the dataset with $40\%$ of data heterogeneity Appendix~\ref{app:datasplit_incomel}.}
    \label{tab:attacks_32}
\end{table}
 From Table~\ref{tab:attacks_32}, we see that our attacks outperform gradient-based ones in both passive and active scenarios across all three datasets. Notably, our passive attack achieves improvements of over $15$ and $8$ percentage points (p.p.) for the Income-L and Medical datasets, respectively. Even when the gradient-based method has access to an oracle, our passive attack still achieves higher accuracy on two datasets and comes very close on Income-A. When shifting to active attacks, the gains are even more substantial. For instance, when the attack is active for 50 rounds, we achieve gains of 13, 22, and 5 percentage points (p.p.) in Income-L, Income-A, and Medical, respectively, over Grad-w-O, and even larger gains over the more realistic Grad. Furthermore, the attack accuracy reaches the performance expected from an adversary who knows the optimal local model. Interestingly, while our attacks consistently improve as the adversary's capacity increases (moving from a passive attacker to an active one and increasing the number of rounds of the active attack), this is not the case for gradient-based methods.

\begin{figure*}[hbt!]
\begin{minipage}[b]{0.33\linewidth}
\includegraphics[scale=0.2]{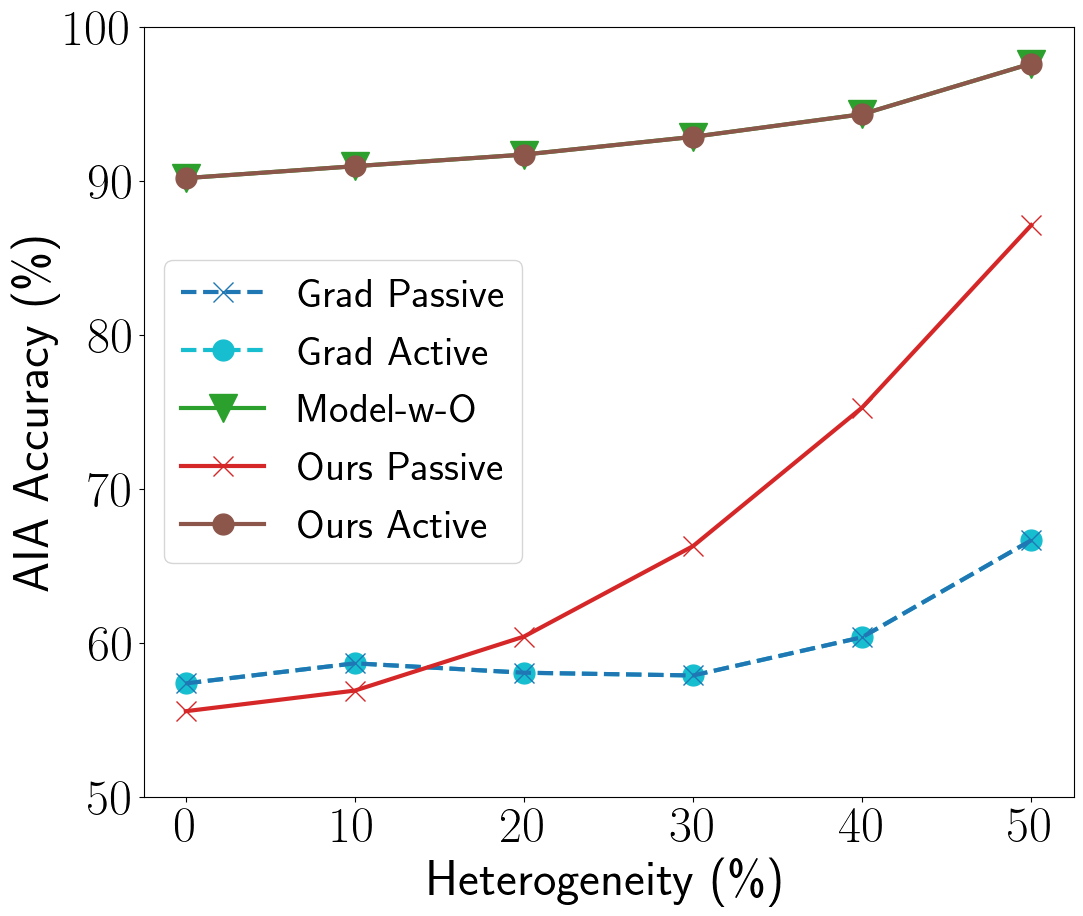}
\end{minipage}
\begin{minipage}[b]{0.33\linewidth}
\includegraphics[scale=0.2]{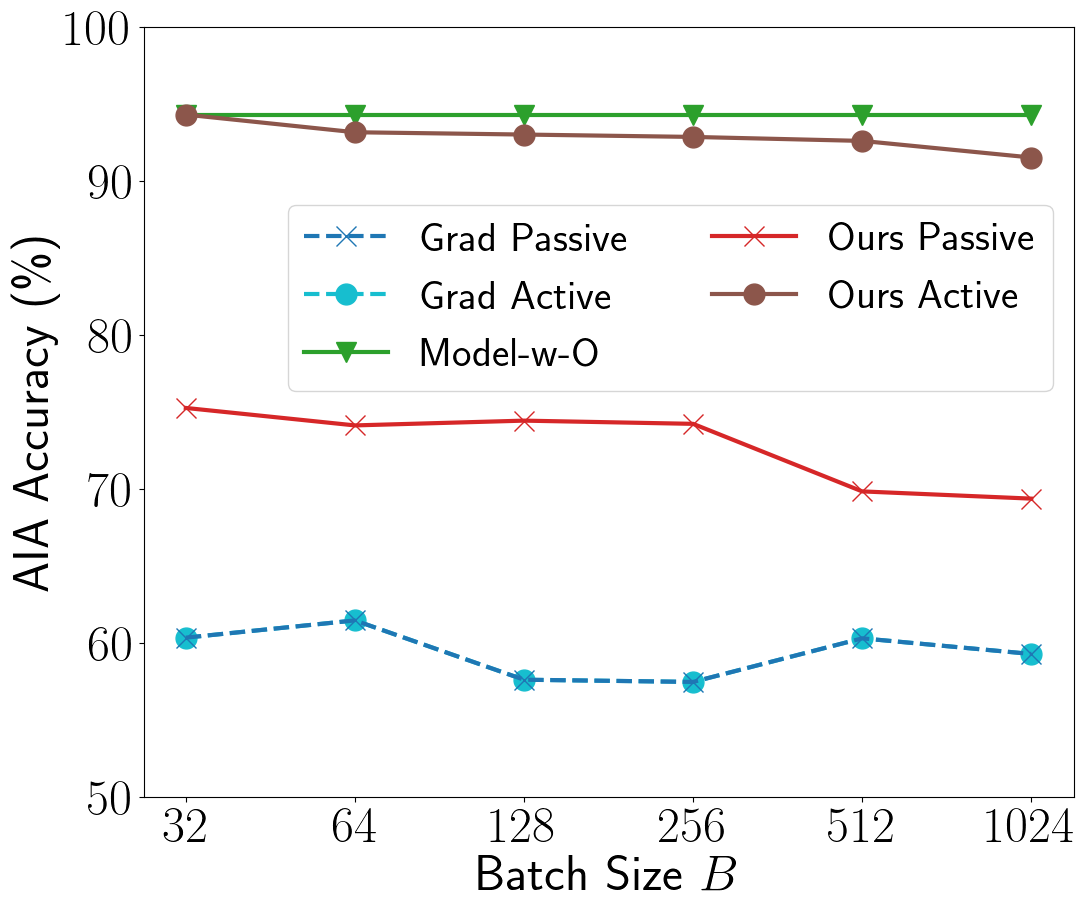}
\end{minipage}
\begin{minipage}[b]{0.33\linewidth}
\includegraphics[scale=0.2]{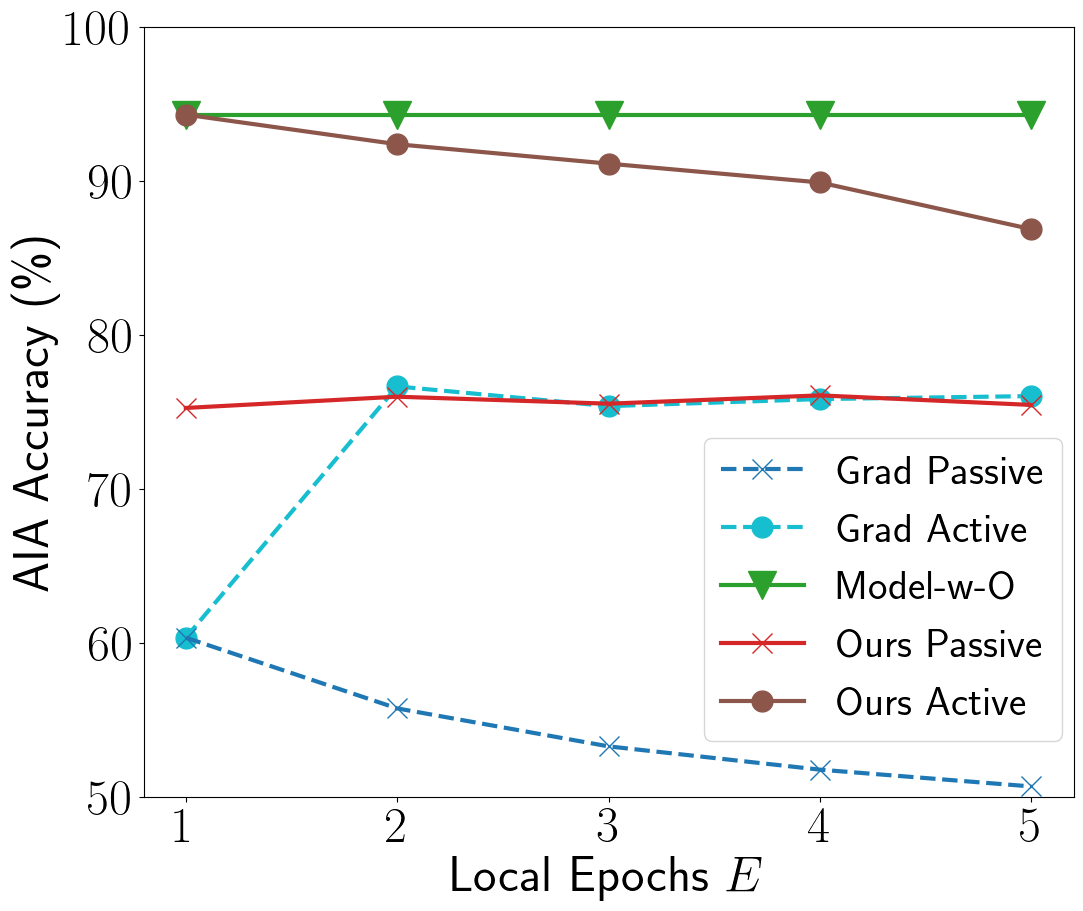}
\end{minipage}
\caption{The AIA accuracy over all clients' local datasets under different heterogeneity levels (left) ($0\%$ represents i.i.d case), batch sizes (center), and local epochs (right) for Income-L dataset. The default values for heterogeneity level, batch size $B$ and local epochs $E$ are set to $40\%$, $32$, and $1$, respectively. The malicious adversary attacks $\lceil 50/E \rceil$ rounds after $\lceil 100/E \rceil $ communication rounds.
Crosses represent passive attacks, while dots represent active attacks. Dashed lines correspond to gradient-based attacks (Grad), and solid lines correspond to model-based attacks (Ours and Model-w-O).}
    \label{fig:impact}
\end{figure*}

\paragraph{Impact of Data Heterogeneity.} 
We simulate varying levels of heterogeneity in the Income-L dataset and illustrate how the attack performance evolves in Figure~\ref{fig:impact} (left). First, we observe that as the data is more heterogeneously split, the accuracy of AIA improves for all approaches. 
Indeed, the data splitting strategy (Appendix~\ref{app:datasplit_incomel}) leads to a greater degree of correlation between the clients' sensitive attributes and the target variable at higher levels of heterogeneity. This intrinsic correlation facilitates the operation of all AIAs.
We observe in these experiments, as in previous ones, that active Grad does not necessarily offer advantages over passive Grad. In the more homogeneous case (which is less realistic in an FL setting), our passive attack performs slightly worse than Grad, but its accuracy increases more rapidly with heterogeneity, resulting in an AIA accuracy advantage of over 20 p.p. in the most heterogeneous case.
Our active attack over 50 additional rounds consistently outperforms Grad by at least 30 p.p. and is almost indistinguishable from \mbox{Model-w-O}.

\paragraph{Impact of Batch Size.} 
Figure~\ref{fig:impact} (center) shows that the performance of our attacks slightly decrease as the batch size increases.  A possible explanation for this is that a larger batch size results in fewer local updates per epoch, leading the client to return a less overfitted model. Despite this, our approach consistently outperforms Grad in all cases.
\paragraph{Impact of the Number of Local Epochs.}
Figure~\ref{fig:impact} (right) shows that under passive attacks our approach is not sensitive to the number of local epochs, whereas Grad's performance deteriorates to the level of random guessing as the number of local epochs increases to 5. Interestingly, active Grad significantly improves upon passive Grad as the number of local epochs increases. While our active approach progressively performs slightly worse, it still maintains an advantage of over 10 p.p.~compared to Grad.

\paragraph{Impact of the Local Dataset Size.}
Figure~\ref{fig:dataset_size_effect} shows how each client's dataset size impacts individual (active) AIA accuracy in Income-A dataset. 
We observe that for all methods, clients with smaller datasets are more vulnerable to the attacks,  
because the models overfit the dataset more easily. 
Moreover, for clients with over 20000 data points, Grad provides an attack performance close to random guessing.
\begin{figure}[ht]
    \centering
    \includegraphics[scale=0.140]{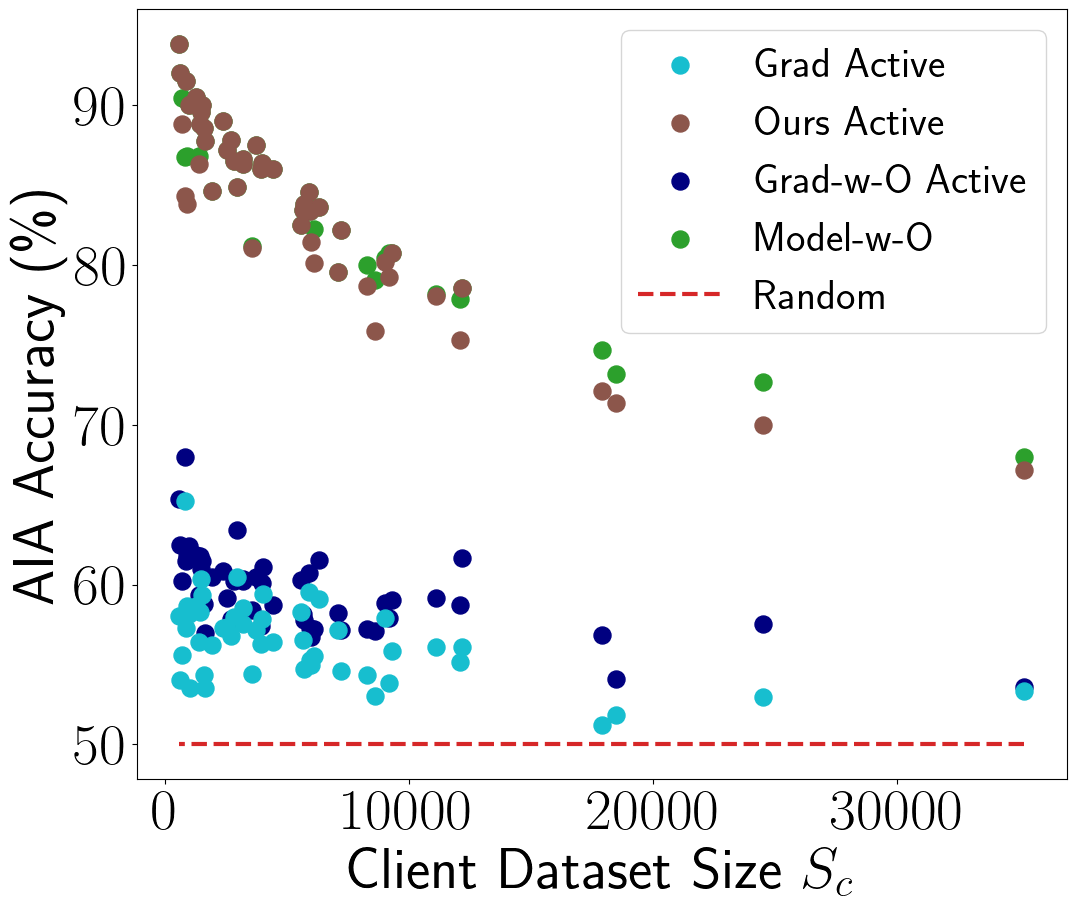}
    \caption{AIA accuracy in Income-A dataset on clients with different local dataset size $S_c$. The experiment setting is the same as in Table~\ref{tab:attacks_32} with  50 active rounds.}
    \label{fig:dataset_size_effect}
\end{figure}

\paragraph{Impact of defense mechanisms.}
To mitigate privacy leakage, we apply a federated version of DP-SGD \cite{DP}, providing sample-level differential privacy guarantees~\cite{Dwork, slDP}. 
Our experiments show that, even with this defense mechanism in place, our approach still outperforms the baselines in most of the scenarios.  Results are in Appendix~\ref{app:defense}.


\section{Discussion and Conclusions}
In our work, we have demonstrated the effectiveness of using model-based AIA for federated \emph{regression} tasks, when the attacker can approximate the optimal local model. For an honest-but-curious adversary, we proposed a computationally efficient approach to reconstruct the optimal \emph{linear} local model in least squares regression. In contrast, for neural network scenarios, our passive approach involves directly utilizing the last returned model (Sec.~\ref{subsec:baseline}). We believe more sophisticated reconstruction techniques may exist, and we plan to investigate this aspect in  future work.

The reader may wonder if the superiority of model-based attacks over gradient-based ones also holds for \emph{classification} tasks. Some preliminary experiments we conducted suggest that the relationship is inverted. We advance a partial explanation for this observation.
 For a linear model with binary cross-entropy loss, it can be shown that the model-based AIA  \eqref{eq:aia_model_based} on a binary sensitive attribute is equivalent to a simple label-based attack, where the attribute is uniquely reconstructed based on the label. This approach leads to poor performance because the attack relies only on general population characteristics and ignores individual specificities.
This observation also holds experimentally for neural networks trained on the Income-L dataset, largely due to the inherent unfairness of the learned models. For example, for the same set of public features, being a man consistently results in a higher probability of being classified as wealthy compared to being a woman. As a consequence, the AIA infers gender solely based on the high or low income label. These preliminary remarks may prompt new interesting perspectives on the relationship between fairness and privacy (see for example the seminal paper~\cite{chang2021privacy}).

Finally, to mitigate privacy leakage, beyond differential privacy mechanisms, there are empirical defenses such as Mixup~\cite{mixup}, and TAPPFL~\cite{aaai_defense}. However, the effectiveness of these defenses has not yet been demonstrated on \emph{regression} tasks, and we leave this for future work.


\section*{Acknowledgements}\label{sec:acks}
This research was supported in part by the Groupe La Poste, sponsor of the Inria Foundation, in the framework of the FedMalin Inria Challenge, as well as by the France 2030 program, managed by the French National Research Agency under grant agreements No. ANR-21-PRRD-0005-01, ANR-23-PECL-0003, and ANR-22-PEFT-0002. It was also funded in part by the European Network of Excellence dAIEDGE under Grant Agreement Nr. 101120726, by SmartNet and LearnNet, and by the French government National Research Agency (ANR) through the UCA JEDI (ANR-15-IDEX-0001), EUR DS4H (ANR-17-EURE-0004), and the 3IA Côte d’Azur Investments in the Future project with the reference number ANR-19-P3IA-0002. \\
The authors are grateful to the OPAL infrastructure from Université Côte d'Azur for providing resources and support.
{\small\bibliography{main}}
\input{sources/appendix}
\end{document}

%% file: sources/appendix.tex
\appendix
\appendixpage

\section{Theoretical Results}
\subsection{Full statement and proof of Theorem~1}\label{app:stochasticproof}
\begin{reptheorem}{prop:stochastic_case}
Consider training a least squares regression through FedAvg (Alg.~\ref{algo:fedavg}) with batch size $B$ and local epochs $E$. Assume that 
\begin{enumerate}
    \item the client's design matrix $\x\in \real^{S_c\times d}$ has rank $d$ equal to the number of features plus one;\footnote{Remember that a client's design matrix has a number of rows equal to the input features of the samples in the client's local training dataset.}
    \item the components of the stochastic (mini-batch) gradient are distributed as sub-Gaussian random variables with variance proxy $\sigma^2$, i.e.,
$
\E[\exp(\ei[i]/\sigma^2)] \leq \exp(1), \forall \theta, \forall i\in\{1,2,...,d\}, 
$
where $\ei = \gradient(\model) - \nabla \locallossfunction(\model)$;\footnote{Note that, this condition implies a bounded variance for each component of the stochastic gradients.}
\item the input model vectors at an observed round are independent of the previous stochastic gradients computed by the attacked client;
\item there exists $\underline{\lambda}>0$ such that $\forall n_c \in \mathbb{N}$, we can always select $n_c$ observation rounds so that $\lambda_{\min}\!\left(\frac{\Theta_{\text{out}}^{T} \Theta_{\text{out}}}{n_c}\right) \ge \underline{\lambda}$, where $\lambda_{\min}(\mathbf{A})$ denotes the smallest eigenvalue of the matrix $\mathbf{A}$, and $\Theta_{\text{out}}$ is defined in Alg.~\ref{algo:lineardecode}.
\end{enumerate}
The error of the reconstructed model $\hat\model^*_c$ of Algorithm~\ref{algo:lineardecode} is upper bounded w.p.  $\ge 1-\delta$  when $\eta\leq \frac{S_c}{2\lambda_{\max}(\x^T\x)}$ and
\begin{equation}
   \left\lVert \hat\model^*_c - \localmodel \right\rVert_2 = \mathcal{O}\left(\eta\sigma d\sqrt{ d E\left\lceil\frac{S_c}{B}\right\rceil \frac{d+1+\ln \frac{2d}{\delta}}{n_c \cdot \underline{\lambda}}} \right). 
\end{equation}
\end{reptheorem}

Before presenting the proof, we discuss the assumptions. The first assumption can be relaxed at the cost of replacing in the proof the inverse $\h^{-1}$ of the matrix $\h = \x^{T} \x$ by its pseudo-inverse $\h^\dagger$. The second assumption is common in the analysis of stochastic gradient methods~\cite{agarwal2012stochastic,pmlr-v70-jin17a,vlaski2021,li2020high,zhou2020convergence}. The third assumption is technically not satisfied in our scenario as the stochastic gradients computed by the attacked client $i$ before an observed round $t_n$ have contributed to determine the models sent back by client $i$ at rounds $t < t_n$ and then the global model $\theta^{t_n}(0)$ sent by the server to client~$i$. Nevertheless, we observe that $\theta^{t_n}(0)$ is almost independent on client-$i$'s  stochastic gradients if the set of clients is very large. Moreover, the assumption is verified if we assume a more powerful adversary who can act as a man in the middle and arbitrarily modifies the model sent to the client. Finally, the fourth assumption corresponds to the ``well-behaved data assumption'' for a generic linear regression problem, which is required to be able to prove the consistency of the estimators, i.e., that they converge with probability $1$ to the correct value as the number of samples diverges (see for example~\cite[Thm.~3.11]{shao2003mathematical}). 
In this context, we can observe that $\lambda_{\min}\!\left(\frac{\Theta_{\text{out}}^{T} \Theta_{\text{out}}}{n_c}\right)=\frac{1}{n_c}\lambda_{\min}\!\left(\Theta_{\text{out}}^{T} \Theta_{\text{out}}\right)=\frac{1}{n_c}\sigma_{\min}^2\!\left(\Theta_{\text{out}}\right)=\frac{1}{n_c}\min_{\mathbf{z}, \lVert \mathbf z\rVert= 1} \lVert \Theta_{\text{out}} \mathbf z\rVert_2^2$. Now, consider that the components of the gradient noise $\e(\theta)$ are independent, each with variance lower bounded by $\tau^2>0$, it follows that $\E[\lVert \Theta_{\text{out}} \mathbf z\rVert_2^2]\ge n_c \tau^2 \frac{d}{d+1} \lVert\mathbf z \rVert_2^2 =n_c \tau^2 \frac{d}{d+1}$. This suggests that $\sigma_{\min}^2\!\left(\Theta_{\text{out}}\right)$  grows linearly with $n_c$ and then it is possible to lower bound $\lambda_{\min}\!\left(\frac{\Theta_{\text{out}}^{T} \Theta_{\text{out}}}{n_c}\right)$ with a positive constant. 


\begin{proof}
Let $\h = \x^T\x$ which is positive definite.
Let $\y\in \real^{S_c}$ be the labels in the local dataset $\localdataset$ with size~$S_c=|\localdataset|$. 
\begin{equation}\label{eq:localloss_linear_regression}
    \locallossfunction(\model) = \frac{\left\lVert \x \model -\y\right\rVert^2}{S_c}
\end{equation}

We know that $\localmodel = (\x^T\x)^{-1} \x^T\y$. When computing the stochastic gradient, we have:
\begin{equation}\label{eq:gradient_sim_e}
    \gradient(\model) =  \nabla \locallossfunction(\model)+ \ei = \frac{2}{S_c}\left( \h \model - \h\localmodel \right) + \ei.
\end{equation}
At round $t$, if selected, client $c$ receives the server model and executes Algorithm~\ref{algo:fedavg}. Let $\model^t_c(k)$ be the model after the $k$-{th} local update. To simplify the notation, in the following, we replace $\mathbf{e}(\model^t_c(k))$ by $\e_k$.
Replacing \eqref{eq:gradient_sim_e} in line~\ref{algoline:gradient_update} of~Algorithm~\ref{algo:fedavg}, we have

\begin{align*}
    \model_c^t(K) &=  \left(\i -  \frac{2\eta}{S_c}\h\right)^{K} \model_c^t(0) +  \left[ \i -  \left(\i -  \frac{2\eta}{S_c}\h\right)^{K} \right]\localmodel \\
      &- \sum_{k=1}^{K}\left(\i -  \frac{2\eta}{S_c}\h\right)^{k-1}\eta\e_{K-k+1}.
\end{align*}

where $K=E\lceil \frac{S_c}{B}\rceil$ as in each epoch there are $\lceil \frac{S_c}{B}\rceil$ local steps.
Let $\w = \left[ \i -  \left(\i -  \frac{2\eta}{S_c}\h\right)^{K} \right]$. $\w$ is proven to be invertible in Lemma~\ref{lem:invert}, then we have:

\begin{align}
   \model_c^t(0) &=\w^{-1}\left(  \model_c^t(0)-\model_c^t(K)\right) + \localmodel \nonumber \\
   &- \w^{-1}  \sum_{k=1}^{K}\left(\i -  \frac{2\eta}{S_c}\h\right)^{k-1}\eta\e_{K-k+1} \label{eq:ols}
\end{align}

  Let $\ga^t = -\eta\w^{-1} \sum_{k=1}^{K}\left(\i -  \frac{2\eta}{S_c}\h\right)^{k-1}\e_{K-k+1}$.  
 Let $\mathbf{x}[i]$ be the $i^{th}$ element of vector $\mathbf{x}$ and $\mathbf{X}[i,:]$ and $\mathbf{X}[:,i]$ be the $i^{th}$ row and column of the matrix $\mathbf{X}$, respectively.
Since every element in $\e$ is sub-Gaussian and the weighted sum of independent sub-Gaussian variables is still sub-Gaussian, $\ga[i]$ is sub-Gaussian with $\E[\ga^t[i]] = 0$ and 
\begin{align*}
\var[\ga^t[i]]& \leq d\eta^2\sigma^2 \lVert \w^{-1} \rVert_2^2  \sum_{k=1}^{K}\left\lVert \left(\i -  \frac{2\eta}{S_c}\h\right)^{k-1}\right\rVert^2_2\\
    &= 
    \frac{d\eta^2\sigma^2}{\lambda^2_{\text{min}}(\w)} \sum_{k=1}^{K}\left\lVert (\i -  \frac{2\eta}{S_c}\h)^{k-1}\right\rVert^2_2\\
    &=  \frac{d\eta^2\sigma^2}{\lambda^2_{\text{min}}(\w)} \sum_{k=1}^{K} \rho_{\max}^2( (\i -  \frac{2\eta}{S_c}\h)^{k-1}),\\
\end{align*}
where $\rho_{\max}(\mathbf{A})$ is the spectral radius of matrix $\mathbf{A}$.

Let $\Lambda$ be a diagonal matrix whose entries are the positive eigenvalues of $\h$ and $\lambda_{\max}(\h)$ be the largest eigenvalue of $\h$, i.e.,~$\lambda_{\max}(\h) = \lVert\Lambda \rVert_1$. When $\eta \leq \frac{S_c}{2\lambda_{\max}(\h)}$, the diagonal values in $\frac{2\eta}{S_c}\Lambda$ are positive and smaller than or equal to $1$. 
Moreover, the matrix  $(\i -  \frac{2\eta}{S_c}\h)^{k-1}$ is positive semi-definite and we have $\rho_{\max}[(\i -  \frac{2\eta}{S_c}\h)^{k-1}] \leq 1$. Thus, we have
\begin{equation}\label{eq:bound_var_noise}
    \var[\ga[i]] \leq \frac{d\eta^2\sigma^2 K}{\lambda^2_{\text{min}}(\w)}. 
\end{equation}

Equation \eqref{eq:ols} for $n_c$ observation can be written in a compact way as
\begin{equation}
\label{eq:ols_matrix}
\Theta_{\text{in}} = \Theta_{\text{out}}(\w^{-1}, \model^*_c)^T + \Gamma^T,
\end{equation}
where $\Gamma = (\gamma^{t_1}, \dots, \gamma^{t_c})$.

Let us denote $\w^{-1}$ as $\v$.
We can determine estimators $(\hat\v, \hat\model^*_c)$ for $(\v, \model^*_c)$, through ordinary least squares (OLS)  minimization. In particular, if we extract the relation involving the i-th row of $\v$ and $\localmodel$, we obtain:
\begin{equation}
\label{eq:ols_row}
\Theta_{\text{in}}[:,i] = \Theta_{\text{out}}(\v[i,:], \model^*_c[i,:])^T + \Gamma^T[:,i].
\end{equation}
The OLS estimator is 
\begin{equation}
(\hat{\v}[i,:],\hat{\model}^*_c[i])^T = \argmin_{\mathbf{v} \in \real^d, \theta \in \real }\left\lVert\Theta_{\text{in}}[:,i] - \Theta_{\text{out}}(\mathbf{v}; \theta) \right\rVert_2^2,
\end{equation}
and can be expressed in closed form as
\begin{equation}
    (\hat{\v}[i,:],\hat{\model}^*_c[i])^T= \left(\Theta_{\text{out}}^T\Theta_{\text{out}}\right)^{-1}\Theta_{\text{out}}^T \Theta_{\text{in}}[:,i].
\end{equation}
A more compact representation is the following
\begin{align*}
    (\hat{\v}, \hat\model^*_c)^T
  =\argmin_{\v\in \real^{d\times d},\model\in \real^{d}} \left\lVert \Theta_{\text{in}}- \Theta_{\text{out}} 
   \left( \w' \,\, \model \right)^T 
     \right\rVert^2_F,
 \end{align*}
and the corresponding closed form is 
\begin{equation} \label{eq:estimate_wv}
     (\hat{\v}, \hat\model^*_c)^T = \left(\Theta_{\text{out}}^T\Theta_{\text{out}}\right)^{-1}\Theta_{\text{out}}^T \Theta_{\text{in}} \in \real^{(d+1)\times d}.
 \end{equation}
Note that this is how $\localmodel$ is estimated in Algorithm~\ref{algo:lineardecode}  line~\ref{algoline:ols}.

The presentation of separate linear systems for each row $i$ in \eqref{eq:ols_row} is also instrumental to the rest of the proof.
We observe that \eqref{eq:ols_row} describes a regression model with independent sub-gaussian noise, and then from Theorem~2.2 and Remark~2.3 in~\cite{mitlecture}, we have that for a given row $i\in \{1, \dots, d\}$, w.p. at least $ 1-\frac{\delta}{d}$, 
\begin{align}\label{eq:expected_error_row_wise}
    (\hat \model^*_c[i]-\localmodel[i])^2
    & \le (\hat \model^*_c[i]-\localmodel[i])^2 + \lVert \hat\v[i,:] - \v[i,:] \rVert_2^2\nonumber\\
    & \le \var[\ga[i]] \frac{d + 1 + \log{\frac{2d}{\delta}} }{n_c \lambda_{\min}\left(\frac{\Theta_{\text{out}}^{T} \Theta_{\text{out}}}{n_c}\right)}. \;\; \nonumber
\end{align}
From \eqref{eq:bound_var_noise} and the fourth assumption, it follows that w.p. at least $1-\frac{\delta}{ d}$
\begin{equation}\label{eq:expected_error_row_wise2}
   (\hat \model^*_c[i]-\localmodel[i])^2
    \le \frac{d\eta^2 \sigma^2 K}{\lambda_{\text{min}}^2(\w)} \cdot \frac{d + 1 + \log{\frac{2 d}{\delta}} }{n_c \underline{\lambda}}\;\; .
\end{equation}
We are going now to consider that each row is estimated through a disjoint set of $n_c$ inspected messages, so that the estimates for each row are independent. We can then apply the union bound and obtain that the Euclidean  distance between the two vectors can be bounded as follows, which concludes the proof:
\begin{equation}\label{eq:bounds_w_v}
 \left\lVert \hat\model^*_c - \localmodel \right\rVert_2^2
    \leq \frac{d^3 \eta^2 \sigma^2 K}{\lambda_{\text{min}}^2(\w)} \cdot \frac{d+1+\ln \frac{2 d}{\delta}}{n_c \underline{\lambda}}, \;\; \text{w.p.} \ge 1 - \delta.
\end{equation}

\end{proof}
\subsubsection{About the Eigenvalues of the Matrix W.}

$\h=\x^{T} \x$ is diagonalizable, i.e., $\h=\mathbf{U} \Lambda \mathbf{U}^{T}$ with $\mathbf{U}$ an orthonormal matrix and $\Lambda$ a diagonal one. After some calculations
$$
\begin{aligned}
\w & =I-\left(I-\frac{2 \eta}{S_c} \h\right)^{E \left\lceil\frac{S_c}{b}\right\rceil} \\
& =\mathbf{U}\left(I-\left(I-\frac{2 \eta}{S_c} \Lambda\right)^{E \left\lceil\frac{S_c}{b}\right\rceil}\right) \mathbf{U}^{T}.
\end{aligned}
$$
Then 
$$
\begin{aligned}
\lambda_{\min}(\w) &=1-\left(1-\frac{2 \eta}{S_c} \lambda_{\min}(\h)\right)^{E \left\lceil\frac{S_c}{b}\right\rceil}\\ 
\lambda_{\max}(\w) &=1-\left(1-\frac{2 \eta}{S_c} \lambda_{\max}(\h)\right)^{E \left\lceil\frac{S_c}{b}\right\rceil}
\end{aligned}
$$

\begin{lemma}\label{lem:invert}
Let $\a$ be a symmetric positive definite matrix. $\i - (\i-\a)^n$ is invertible for $n\geq 0$.
\end{lemma}

\begin{proof}
Since $\a$ is positive definite and symmetric, it can be decomposed to $\u\Lambda\u^T$, where $\u\u^T = \i$ and $\Lambda$ is a diagonal matrix whose entries are the positive eigenvalues of $\a$. Thus we have
\begin{align*}
    \i-(\i-\a)^n &= \u\i\u^T - (\u\i\u^T - \u\Lambda\u^T)^n \\
    & =\u\i\u^T - \u(\i -\Lambda)^n\u^T\\
    &= \u (\i-(\i-\Lambda)^n)\u^T.
\end{align*}
Since $\Lambda$ is with only positive values on diagonal, eigenvalues of $\i-(\i-\a)^n$ are non-zeros. So  $\i-(\i-\a)^n$ is invertible and its inverse is $\u (\i-(\i-\Lambda)^n)^{-1}\u^T$.
\end{proof}

\subsection{Proof of Proposition~2}\label{app:proof_corollary}
\begin{proof}
When FedAvg is executed with full batch (i.e., $\ei = \mathbf{0}$), according to Eqs.~\ref{eq:gradient_sim_e} and~\ref{eq:ols}, we have
\begin{align}
        \gradient(\model) =  \nabla \locallossfunction(\model)= \frac{2}{S_c}\left( \h \model - \h\localmodel \right) \nonumber \\ 
          \model_c^t(0) =\w^{-1}\left(  \model_c^t(0)-\model_c^t(K)\right) + \localmodel. \label{eq:ols_no_noise}
\end{align}

Then once $n_c>d$, by solving $d$ multivariate linear equations (Eq.~\ref{eq:ols_no_noise}), we obtain the exact optimal local model $\localmodel$. 
\end{proof}

\subsection{Proof of Proposition~3}\label{app:proof_proposition}

\begin{proof}
To prove the lower bound for the communications,  we construct a specific ``hard" scenario. This scenario is inspired by the lower bound for the convergence of gradient methods minimizing smooth convex functions~\cite[Sect.~2.1.2]{nesterov2003introductory}.   

Let $\cc$ be the client having the local dataset ($\x$, $\y$) such that $\h_\cc= \x^T\x$ is a tridiagonal matrix with $\h_\cc[i,i+1] = \h_\cc[i+1,i] = -1/2$ and $\h_\cc[i,i] = 1, \forall i\in \{1,2,...,d\}$, and $\x^T\y=[1/2,0,0,...,0]^T$.\footnote{One example for the local data ($\x$, $\y$) satisfying the condition: $\x\in\real^{d\times d}$ with $\x[1, 1] = 1; \x[i+1,i+1] =\sqrt{1-\frac{1}{4\x[i, i]^2}}$ and $\x[i+1, i]=\frac{-1}{2\x[i, i]}, \forall i\in \{1, ..., d-1\}$; $\x[i,i+1]=\x[i, j]=0, \forall |i-j|\geq 1$. $\y\in\real^{d}$ with $\y[1] = \frac{1}{2}$ and $\y[i+1]= \frac{-\x[i+1, i]\y[i]}{\x[i+1, i+1]},  \forall i\in \{1, ..., d-1\}$.} We have
\begin{eqnarray}
\locallossfunctioncc(\model) &=& \frac{1}{S_c} \left(\model^T\h_\cc\model-2(\x^T\y)^T\model - \y^T\y \right) \nonumber\\
&=& \frac{1}{S_c} \left(\sum_{i=1}^{d} \model_{[i]}^2 - \sum_{i=1}^{d-1} \model_{[i]}\model_{[i+1]} -\model_{[1]} - C \right) \label{eq:loss_function_design}
\end{eqnarray}
where $C=\y^T\y$ is a constant and
$\localmodelcc =\arg\min \locallossfunctioncc(\model) =  (\h_\cc)^{-1}\x^T\y = [1-\frac{1}{d+1}, 1-\frac{2}{d+1}, ..., 1-\frac{d}{d+1}]^T$~\cite{nesterov2003introductory}. According to~\eqref{eq:loss_function_design}, we know that if $\model_{[i]}=\model_{[i+1]}=\ldots=\model_{[d]}=0$, then
$\nabla\locallossfunctioncc (\model)_{[i+1]}=\ldots=\nabla\locallossfunctioncc (\model)_{[d]} = 0$.

At the same time, for any other client $\forall \bar c\in \clients \setminus \{\cc\}$, we assume their local optimum are zeros where their local dataset $\mathbf{X}_{\bar \cc}=\mathbf{I}$ and $\mathbf{y}_{\bar \cc}  = \mathbf{0}$.
In this setting, we know that if the $i^{th}$ element of the global model is zero, i.e.,~$\model(t)_{[i]}=0$, then $\model^t_{\bar c}(E)_{[i]}=0$. 

Let $n = |\clients|$ be the number of clients and assume that every client has $S_c$ local data samples. Then the global empirical risk and its gradients have the following expressions: 
\begin{align}
\globallossfunction(\model) &= \frac{1}{nS_c} \left(\sum_{i=1}^{d} n\times \model_{[i]}^2 - \sum_{i=1}^{d-1} \model_{[i]}\model_{[i+1]} -\model_{[1]} - C' \right), \label{eq:global_function_design} \\
\nabla \globallossfunction(\model)_{[1]} &=  
\frac{1}{nS_c}(2n\times \model_{[1]} - \model_{[2]} - 1),\label{eq:gradient1}\\
\nabla \globallossfunction(\model)_{[i]} &=\frac{1}{nS_c} \left( 2n\times \model_{[i]} - \model_{[i-1]} - \model_{[i+1]} \right), \forall i\in\{2,...,d-1\},\label{eq:gradient2} \\
\nabla \globallossfunction(\model)_{[d]} &= \frac{1}{nS_c} \left( 2n\times \model_{[d]} - \model_{[d-1]} \right) \label{eq:gradient3},
\end{align}
where $C'$ is a constant. According to~\eqref{eq:gradient1},~\eqref{eq:gradient2} and~\eqref{eq:gradient3}, the global optimum $\model^*$ satisfies:
\begin{eqnarray}
\model^*_{[1]} &=& (1+\model^*_{[2]})/2n, \label{eq:optimum_1}\\ 
\model^*_{[i]} &=& 2n\model^*_{[i+1]} - \model^*_{[i+2]}, \forall i\in\{1,...,d-2\},\label{eq:optimum_2}\\
\model^*_{[d-1]}  &=& 2n\model^*_{[d]}.\label{eq:optimum_3}
\end{eqnarray}
Equations~\ref{eq:optimum_2} and~\ref{eq:optimum_3} show that $\model^*_{[i]}$ is proportional to $\model^*_{[d]}$ for every $i\in\{1,...,d-1\}$, i.e.,~$\model^*_{[i]} = k_i\times \model^*_{[d]}$ where $k_i>0$. 
Since $\model^*_{[1]} = k_1\times \model^*_{[d]}$ and $\model^*_{[2]} = k_2\times \model^*_{[d]}$ where $k_1>0$ and $k_2>0$, by substituting $\model^*_{[1]}$ and  $\model^*_{[2]}$ into~\eqref{eq:optimum_1}, we can see that $\model^*_{[d]} \not = 0$.
Therefore, we can prove then that every element of the global optimum is non-zero, i.e.,~$\model^*_{[i]} \not = 0, \forall i\in \{1,..,d\}$.

Now, suppose that we run the FedAvg with one local step under the above scenario, with initial global model $\model(0) = \mathbf{0}$, i.e.,~$\model^0_c(0) = \mathbf{0}, \forall c\in \clients$. 
According to the previous analysis, we know that~$\forall t\in \{0,1...,d-1\}$:
\begin{eqnarray}
\model^t_\cc(E)_{[t+1]} = \model^t_\cc(E)_{[t+2]} = \ldots = \model^t_\cc(E)_{[d]} = 0,
\end{eqnarray}
and
\begin{eqnarray}
\model(t)_{[t+1]} = \model(t)_{[t+2]} = \ldots = \model(t)_{[d]} = 0.
\end{eqnarray}
Therefore, to reach the non-zeros global optimum, the client $\cc$ needs to communicate with the server (be selected by the server) at least $d$ times. 

Moreover, to recover the local optimum of client $\cc$, the adversary must listen on the communication channel for at least $d$ times.  
In fact, suppose that the client $\cc$ holds another local data set which gives $\h_{\cc}'$ equals to $\h_{\cc}$ but for the last row and the last column that are zeros. Under this case, the adversary will have the same observation under $\h_{\cc}$ and under $\h_{\cc}'$ till round $d-1$.  
\label{proof_com}
\end{proof}

\section{The details of experimental setups}
\subsection{Client's update rule}\label{app:client_update}
In all our experiments we assume each client follows FedAvg update rule, performing locally multiple stochastic gradient steps. The detailed update rule is in Alg~\ref{algo:fedavg}. 
\begin{algorithm}[!tbh]
\caption{Client $c$'s local update rule ($\localupdaterule^c(\model^t, \localdataset)$) in FedAvg~\cite{mcmahan2017communication}}\label{algo:fedavg}
\textbf{Input}: server model $\model^t$, local dataset $\localdataset$, batch size $B$, number of local epochs $E$, learning rate~$\eta$.
\begin{algorithmic}[1]
        \STATE{$\model_c^t(0) \leftarrow \model^t$, $\mathcal{B} \leftarrow$ (split $\localdataset$ into batches of size $B$), $k\leftarrow 0$}
        \FOR{each local epoch $e$ from 1 to $E$}
          \FOR{batch $b \in \mathcal{B}$}
                \STATE{$\model_c^t(k+1) \leftarrow \model_c^t(k) - \eta \times \gradient (\model_c^t(k),b)$, \,\,\,
              where $\gradient(\model_c^t(k), b) = \frac{1}{B} \sum_{x\in b} \nabla \ell(\model_c^t(k), x)$} \label{algoline:gradient_update}
                \STATE{$k\leftarrow k+1$}
            \ENDFOR
        \ENDFOR
    \STATE{Return $\model_c^t(K)$, where $K=E\lceil\frac{S_c}{B}\rceil$}
\end{algorithmic}

\end{algorithm}
\subsection{Experiment in Figure~\ref{fig:gradient_vs_model}}\label{app:figure1}
We evaluate different AIAs when four clients train a neural network (a single hidden layer of 128 neurons using ReLU as activation function) through FedAvg with 1 local epoch and batch size 32. Each client stores $S_c=|\localdataset|$ data points randomly sampled from ACS Income dataset~\cite{income_dataset}. This dataset contains census information from 50 U.S. states and Puerto Rico, spanning from 2014 to 2018. It includes 15 features related to demographic information such as age, occupation, and education level. 
The adversary infers the gender attribute of every data sample held by the client given access to the released (public) information.

The \emph{regression} task is to predict an individual's income. The \emph{classification} task is to predict whether an individual has an income higher than the median income of the whole dataset, i.e., 39K$\$$.
To train the classification task, for each sample $i$ with target income $y_i$, the label is given by $\mathds{1}_{y_i>39K}$. 
\subsection{Toy dataset setting}\label{app:toydataset}
We test our proposed passive LMRA~(Alg.~\ref{algo:lineardecode}) on a toy federated least squares regression with two clients. Each client $c$ has 1024 samples with 10 features, where nine numerical features are sampled from a uniform distribution over $[0,1)$ and one binary feature is sampled from a Bernoulli distribution with $p=1/2$. The prediction $\mathbf{y}$ is generated from the regression model $\mathbf{y} = \x\model^*+\mathbf{\epsilon}$ where $\mathbf{\epsilon}$ is drawn from $\mathcal{N}(0, 0.1)^{S_c}$ and the optimal local model $\model^*\in \real^{d}$ is drawn from the standard normal distribution where $d=11$. 
The training is run by FedAvg with 1 local epoch, batch size $64$, $256$, and $1024$, and 5 seeds each, respectively. The number of communication rounds are 300 and  the honest-but-curious adversary only eavesdropped $d+1=12$ messages from rounds $\mathcal{T}=\{i*20 | i\in \{0,1,...,11\}\}$.

Note that, for the full batch scenario (i.e.,~$B=S_c$), numerical inversion errors occurred in line~\ref{algoline:ols} may prevent the exact computation of the local model, particularly when $\Theta_{\text{out}}$ is ill-conditioned. That explains why in Figure~\ref{fig:linear_toy}, we have $||\hat\model^*-\model^*||_2$ close to but not exactly equal to zero when $B=S_c$.\footnote{In real practice, the adversary can optimize the set of the messages considered to minimize the condition number of $\Theta_{\text{out}}$.}

\subsection{Data splitting strategy for Income-L dataset }\label{app:datasplit_incomel}
For Income-L dataset, we apply a splitting strategy to control statistical heterogeneity among the 10 clients, which is detailed in Alg.~\ref{algo:Income-L}.

To achieve this, we first partition the initial dataset into two clusters, $\mathcal{D}_{h}$ and $\mathcal{D}_{l}$, which exhibit strong opposing correlations between the sensitive attribute and the target income. More precisely, $\mathcal{D}_h$ contains samples of rich men and poor women and $\mathcal{D}_{l}$ contains samples of poor men and rich women (lines~2-3).  
We then randomly select $ \min(|\mathcal{D}_{h}|,|\mathcal{D}_{l}|)$ samples from each cluster to have balanced size, denoted by $\mathcal{D}'_h$ and $\mathcal{D}'_h$ (lines~5-6). 
Initially, $\mathcal{D}'_h$ and $\mathcal{D}'_l$ have clearly different distributions. By randomly swapping a fraction of $0.5-h$ of samples between the two clusters (lines~7-8),  the distributions of $\mathcal{D}'_h$ and $\mathcal{D}'_l$ become more similar as $h$ decreases. Lastly, each cluster is divided equally into datasets for five clients (lines~11-12).


\begin{algorithm}[t]
\caption{Splitting strategy for Income-L dataset with heterogeneity level $h \in [0, 0.5]$}\label{algo:Income-L}
 \textbf{Input}: the initial Income-L dataset $\mathcal{D}=\{(\mathbf{x}^p(i),s(i), y(i)), i=1,..., |\mathcal{D}|\}$
\begin{algorithmic}[1]
    \STATE $\mathrm{med} \leftarrow$ median value in $\{y(i), i=1,..., |\mathcal{D}|\}$
    \STATE Let $\mathcal{D}_{h}= \{((\mathbf{x}^p(i),s(i), y(i)) \in \mathcal{D}\,\,| ( s(i) = \mathrm{man} \land y(i) > \mathrm{med} )\lor ( s(i) = \mathrm{woman} \land y(i) <= \mathrm{med}) \}$
    \STATE Let $\mathcal{D}_{l} = \mathcal{D} \setminus \mathcal{D}_{h}$
    \STATE $k \leftarrow \min(|\mathcal{D}_{h}|,|\mathcal{D}_{l}|)$
    \STATE $\mathcal{D}_{h}' \leftarrow$ sample randomly $k$ points from $\mathcal{D}_{h}$
    \STATE $\mathcal{D}_{l}' \leftarrow$ sample randomly $k$ points from $\mathcal{D}_{l}$
    \STATE $\mathcal{D}_{hs}' \leftarrow$ sample randomly $(0.5-h)k$ points from $\mathcal{D}_{h}'$
    \STATE $\mathcal{D}_{ls}' \leftarrow$ sample randomly $(0.5-h)k$ points from $\mathcal{D}_{l}'$
    \STATE $\mathcal{D}_{h}'\leftarrow (\mathcal{D}_{h}' \setminus \mathcal{D}_{hs}') \cup \mathcal{D}_{ls}'$
    \STATE $\mathcal{D}_{l}'\leftarrow$ $(\mathcal{D}_{l}' \setminus \mathcal{D}_{ls}') \cup \mathcal{D}_{hs}'$
    
    \STATE Split $\mathcal{D}_{l}'$  equally among the first 5 clients 
    \STATE Split  $\mathcal{D}_{h}'$ equally among the last 5 clients 

\end{algorithmic}
\end{algorithm}

\subsection{Hyperparameters}\label{app:hyperparameters}
\paragraph{Learning rates for FL training on neural network}
In the experiments on Income-L, clients train their model using varying batch sizes from the set $\mathcal{B} = \{32, 64, 128, 256, 512, 1024\}$. The learning rates are tuned in the range $[1\cdot10^{-7}, 5\cdot 10^{-6}]$, according to the batch size: $5 \cdot 10^{-7}$ for batch sizes of 32 and 64, $1 \cdot 10^{-6}$ for a batch size of 128, $2 \cdot 10^{-6}$ for batch sizes of 256 and 512, and $3 \cdot 10^{-6}$ for a batch size of 1024. We keep the same learning rates when varying the number of local epochs and the data heterogeneity level.
For Income-A,  the learning rate is set to $1\cdot 10^{-6}$.
For Medical, the learning rate is set to $2\cdot10^{-6}$. \\
\paragraph{Hyperparamters for Gradient-based attacks}

 For the passive adversary, the set of inspected messages $\mathcal T$ is selected in  $T^p = \{ \{\text{first } \max\{1, \lfloor f |\mathcal{T}_c| \rfloor\}$  rounds in  $\mathcal{T}_c\}$, for  $f\in \mathcal{F}\}$. 
For the active adversary, $\mathcal T$ is selected in  $T^p \cup \{ \{\text{first } \max\{1, \lfloor f |\mathcal{T}^a_c|\rfloor\}  \text{ rounds in } \mathcal{T}^a_c\}, \text{ for } f\in \mathcal{F}\}$. We set $\mathcal{F} = \{ 0.01, 0.05, 0.1, 0.2, 0.5, 1\}$

To solve~\eqref{eq:aia_gradient_based} with the Gumbel-softmax trick, we set the Gumbel-softmax temperature to 1.0 and use a SGD optimizer with a learning rate tuned from the set $\{10^n, n= 2,3,\dots,6\}$ for every attack.

\paragraph{Hyperparameters for our active attack}
 
In our attack with an active adversary, detailed in Alg.~\ref{algo:active}, Adams' hyperparameters are selected following a tuning process performed using Optuna \citep{optuna} hyperparameter optimization framework. In all the experiments, values for the learning rate, $\beta_1$ and $\beta_2$ are optimized to minimize each client's training loss. Learning rate is tuned in range $[0.1, 50]$, whereas $\beta_1$ and $\beta_2$ values are tested in range $[0.6, 0.999]$. The optimal hyperparameters' set is determined after 50 trials.



\section{Additional experimental results}
\subsection{Federated least squares regression}\label{app:linear_model}
Here, we show the results of AIAs for federated least squares regression task, on Income-L, Income-A and Medical datasets. In all the experiments, each client trains a linear model for 300 communication rounds with 1 local epoch and batch size 32. The learning rate is set to $5\cdot10^{-3}$.  The passive adversary may eavesdrop all the exchanged messages until the end of the training. The active adversary launches the attack after 300 rounds for additional 10 and 50 rounds. All attacks are targeted at a randomly chosen single client.
To approximate the optimal local model of the targeted client, our passive adversary applies Alg.~\ref{algo:lineardecode} and  our active adversary applies Alg.~\ref{algo:active}.

To reduce the task difficulty by a linear model, we shrink the feature space of Income-L and Income-A. More precisely, we remove ‘Occupation', ‘Relationship', and ‘Place of Birth' features, transform ‘Race Code' and ‘Marital Status' from categorical to binary features (i.e., White/Others and Married/Not married), and reduce the cardinality of ‘Class of Worker' to 3 groups (Public employess/Private employees/Others). 

\paragraph{Hyperparameters}
We optimize the set of messages for both our method and gradient-based passive attacks. 
For gradient-based attacks,  the set of inspected messages $\mathcal{T}$ is selected in $T^p = \{ \{\text{first t}$  rounds in  $\mathcal{T}_c\}$, for  $t\in \{1,5,10,20,50,100,150,300\}\}$. 
For our attack, we explore a much larger number of sets of messages as our attack is computationally-light, evolving only simple matrix computation. In particular, we randomly sample $10^7$ sets of size $d+1$ and choose the one which minimizes the condition number of $\Theta_{\text{out}}$, since an ill-conditioned matrix  $\Theta_{\text{out}}$ leads to a high numerical inversion error as mentioned in App.~\ref{app:toydataset}.  All other hyperparameters are consistent with those used in the neural network experiments.

\begin{table}[ht]
    \centering
    \begin{tabular}{|c|c|c|c|c|}
        \hline
         \multicolumn{2}{|c|}{\bfseries \backslashbox{AIA (\%)}{Datesets}}  & Income-L  &Income-A  &Medical \\
         \hhline{|=|=|=|=|=|}
    \multirow{3}{*}{\bfseries Passive}& Grad&53.10\tiny{$\pm$1.40} &49.74\tiny{$\pm$3.17} &87.76\tiny{$\pm$3.80} \\ 
    & Grad-w-O &58.19\tiny{$\pm$0.41}  &55.97\tiny{$\pm$0.38} &\bf{94.68}\tiny{$\pm$0.23} \\ 
    &Ours &\bf{59.05}\tiny{$\pm$0.40} &\bf{56.56}\tiny{$\pm$0.41} &94.13\tiny{$\pm$0.16}\\
    \hhline{|=|=|=|=|=|}
    \multirow{3}{*}{\shortstack[c]{\bfseries Active \\(10 Rnds)}}& Grad&53.10\tiny{$\pm$1.40} &49.74\tiny{$\pm$3.17} &87.76\tiny{$\pm$3.80} \\ 
       & Grad-w-O &\bf{59.06}\tiny{$\pm$0.07} &\bf{57.18}\tiny{$\pm$0.39} &\bf{94.68}\tiny{$\pm$0.23} \\ 
    & Ours&57.99\tiny{$\pm$1.71} &56.44\tiny{$\pm$0.12} & 90.47\tiny{$\pm$5.12}\\
    \hline
    \multirow{3}{*}{\shortstack[c]{\bfseries Active \\(50 Rnds)}}& Grad&53.10\tiny{$\pm$1.40} &49.74\tiny{$\pm$3.17} &87.76\tiny{$\pm$3.80} \\ 
     & Grad-w-O &\bf{60.69}\tiny{$\pm$0.30} &\bf{57.18}\tiny{$\pm$0.39} &\bf{94.68}\tiny{$\pm$0.23} \\ 
    & Ours&59.10\tiny{$\pm$0.16} &56.22\tiny{$\pm$0.06} &93.91\tiny{$\pm$0.08}\\
    \hhline{|=|=|=|=|=|}
    \multicolumn{2}{|c|}{\bfseries Model-w-O}& 59.10\tiny{$\pm$0.16} &56.56\tiny{$\pm$0.41} &94.13\tiny{$\pm$0.16} \\ 
    \hline 
    \end{tabular}
    \caption{The AIA accuracy over one (random chosen) targeted client local dataset, evaluated under both honest-but-curious (passive) and malicious (active) adversaries across Income L, Income-A and Medical FL datasets. The standard deviation is evaluated over three FL training runs. All clients solve a \emph{least squares regression} problem running FedAvg with 1 epoch and batch size 32. For income-L, we consider an i.i.d setting.}
    \label{tab:attack_linear}
\end{table}

Results in Table~\ref{tab:attack_linear} show that our attacks outperform the baseline Grad in both passive and active scenarios on across all three datasets. Notably, our passive attack achieves improvements of over 4, 5 and 7 percentage points (p.p.) for Income-L, Income-A and Medical datasets, respectively. Even when the gradient-based method has access to an oracle, our passive attacks still achieves higher accuracy on Income-A and Income-L and comes very close on Medical dataset. 
More importantly, our passive attack reaches already  to the performance expected from an adversary who knows the optimal local model,  demonstrating the effectiveness of our passive approach in approximating the optimal local model (Alg.~\ref{algo:lineardecode}).
When shifting to active attacks,  the performance gain over Grad remains largely consistent after 50 active rounds.

\subsection{Impact of the starting round for active attack}
\label{app:starting_points}
In our implementations, with a local epoch of $1$, all active attacks were initiated from the $100^{\text{th}}$ communication round. Here, we vary the starting round of the active attacks and run the attacks for additional 50 rounds.  We illustrate the attacks' performance in Figure~\ref{fig:attack_starting_point_effect}. 
We observe that our approach shows slight improvements when the attack is launched in the later phases of training. However, the gradient-based method Grad is more effective during the early training phase. Overall, our approach maintains an advantage of over 12 percentage points.  
\begin{figure}
    \centering
    \includegraphics[scale=0.3]{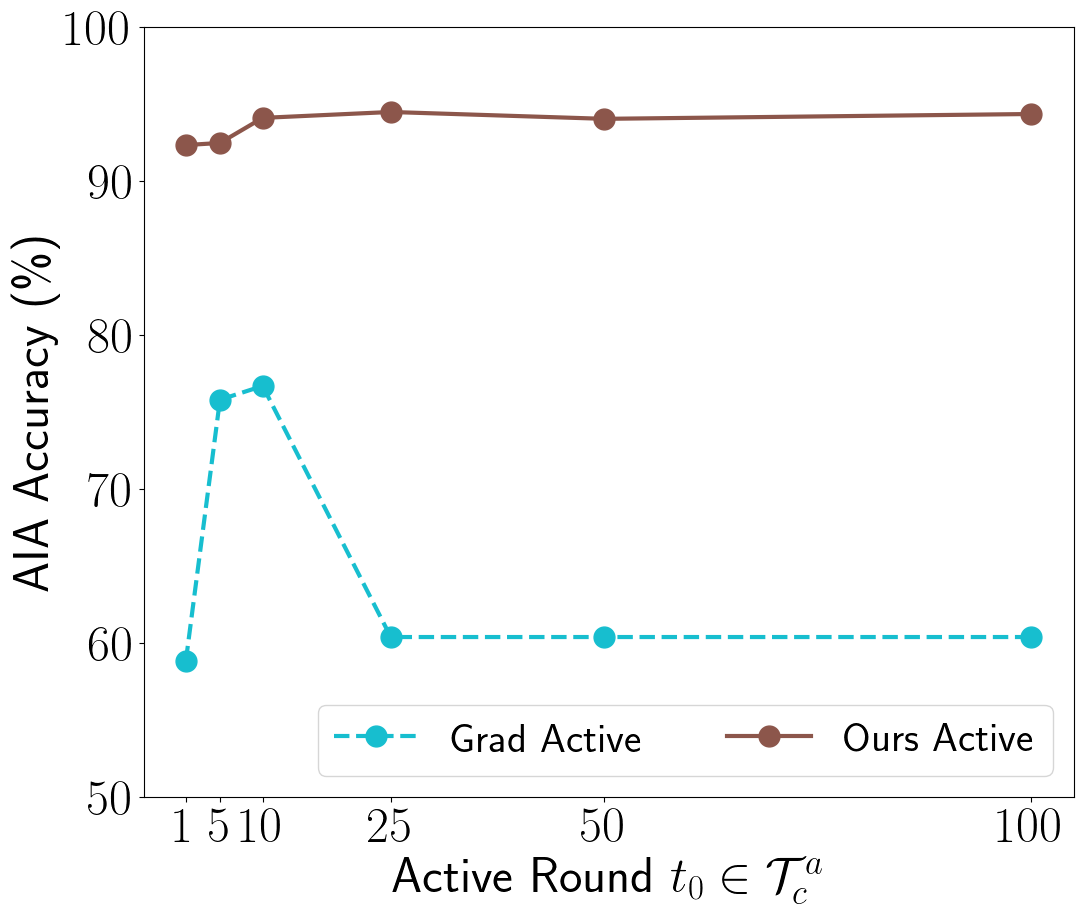}
    \caption{The AIA accuracy over all clients' local datasets under different starting points of active attack for Income-L dataset ($40\%$ heterogeneity level). The clients train a neural network through FedAvg with 1 local epoch and batch size 32.}
    \label{fig:attack_starting_point_effect}
\end{figure}

\subsection{Impact of the number of active rounds} 
\label{app:impact_active_rounds}
We evaluate the performance of active attacks under different numbers of active rounds $|\mathcal{T}^a_c|$ (Figure~\ref{fig:attack_round_effect}), during which an active adversary launches attacks.  Our attack becomes more powerful as the number of active rounds increases, whereas Grad does not demonstrate the same level of effectiveness.

\begin{figure}
    \centering
    \includegraphics[scale=0.3]{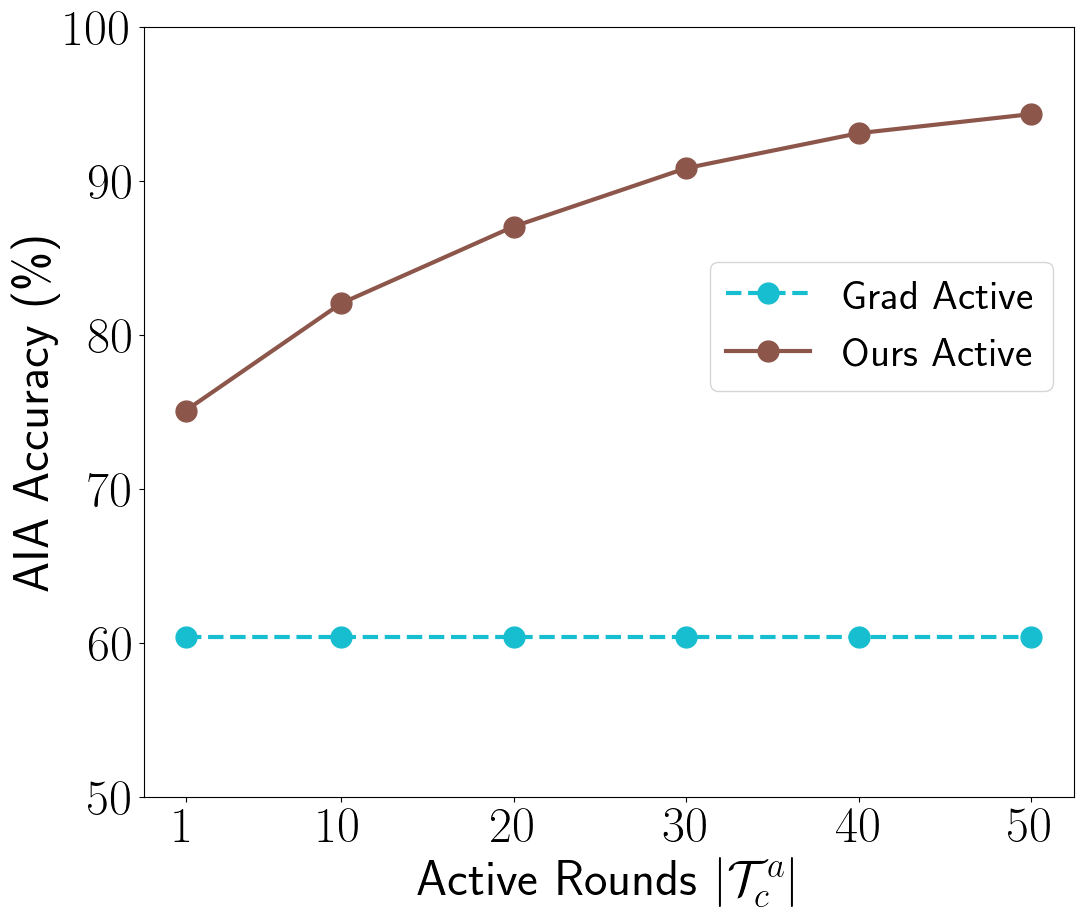}
    \caption{The AIA accuracy over all clients' local datasets under different numbers of active rounds for Income-L dataset ($40\%$ heterogeneity level). The clients train a neural network through FedAvg with 1 local epoch and batch size 32.}
    \label{fig:attack_round_effect}
\end{figure}

\subsection{Defense}
\label{app:defense}
To mitigate privacy leakage, we apply a federated version of DP-SGD~\cite{DP}, which provides ($\epsilon, \delta$) sample-level differential privacy guarantees~\cite{Dwork, slDP}. More precisely, the clients clip and add Gaussian noises to their gradients in FL.  We used Opacus~\citep{Opacus} to incorporate a $(1,1\cdot10^{-5})$-differentially private defense on Income-L and Medical datasets, and $(1,1\cdot10^{-6})$-differentially private defense on Income-A.
All attacks are targeted at a randomly chosen single client. Other experimental settings are consistent with those used in the neural network experiments. 

\paragraph{Hyperparameters for defense}
We adjust the clipping norm for each dataset and select the one that yields the lowest validation loss in the final global model.
For Income-L and Income-A datasets, the clipping norm is tuned over the set $\{1\cdot 10^6, 3\cdot 10^6, 5 \cdot 10^6, 7\cdot 10^6, 9\cdot 10^6\, 1\cdot 10^7, 3\cdot 10^7, 5 \cdot 10^7\}$. For Medical dataset,  the clipping norm is tuned over the set $\{5 \cdot 10^5, 7\cdot 10^5, 9\cdot 10^5, 1\cdot 10^6, 3\cdot 10^6, 5 \cdot 10^6, 7\cdot 10^6, 9\cdot 10^6 \}$. \\


\begin{table}[ht]
    \centering
    \begin{tabular}{|c|c|c|c|c|}
        \hline
         \multicolumn{2}{|c|}{\bfseries \backslashbox{AIA (\%)}{Datasets}}  & Income-L  &Income-A  &Medical \\
         \hhline{|=|=|=|=|=|}
    \multirow{3}{*}{\bfseries Passive}& Grad&59.34\tiny{$\pm3.58$} &50.52\tiny{$\pm2.49$} &63.51\tiny{$\pm6.63$} \\ 
    & Grad-w-O &\bf{77.67}\tiny{$\pm1.01$} &53.83\tiny{$\pm0.19$} &91.09\tiny{$\pm0.08$} \\ 
    &Ours &48.69\tiny{$\pm4.03$}  &\bf{58.08}\tiny{$\pm0.11$}  & \bf{94.19}\tiny{$\pm0.23$}\\
    \hhline{|=|=|=|=|=|}
    \multirow{3}{*}{\shortstack[c]{\bfseries Active \\(10 Rnds)}}& Grad&59.34\tiny{$\pm3.58$}  &50.52\tiny{$\pm2.49$} &63.51\tiny{$\pm6.63$} \\ 
       & Grad-w-O &\bf{77.67}\tiny{$\pm1.01$} &54.42\tiny{$\pm2.77$} &91.09\tiny{$\pm0.08$} \\ 
    & Ours&62.04\tiny{$\pm$3.16} &\bf{57.60}\tiny{$\pm0.85$} &\bf{93.96}\tiny{$\pm0.41$}\\
    \hline
    \multirow{3}{*}{\shortstack[c]{\bfseries Active \\(50 Rnds)}}& Grad&59.34\tiny{$\pm3.58$}  &50.52\tiny{$\pm2.49$} &63.51\tiny{$\pm6.63$} \\ 
     & Grad-w-O &\bf{77.67}\tiny{$\pm1.01$}  &54.42\tiny{$\pm2.77$} &91.09\tiny{$\pm0.08$} \\ 
    & Ours&71.37\tiny{$\pm1.86$} &\bf{57.25}\tiny{$\pm0.86$} &\bf{94.30}\tiny{$\pm0.08$} \\
    \hhline{|=|=|=|=|=|}
    \multicolumn{2}{|c|}{\bfseries Model-w-O}&79.33\tiny{$\pm1.11$} &58.53\tiny{$\pm0.61$} &94.30\tiny{$\pm0.08$} \\ 
    \hline 
    \end{tabular}
    \caption{The AIA accuracy over one (random chosen) targeted client's local dataset evaluated under both honest-but-curious (passive) and malicious (active) adversaries across Income L, Income-A and Medical FL datasets. The standard deviation is evaluated over three FL training runs. All clients train a neural network through a federated version of DP-SGD with 1 local epoch and batch size 32, providing ($\epsilon=1$, $\delta=1\cdot 10^{-5}$) sample level differential privacy for every client on Medical and Income-L, and ($\epsilon=1$, $\delta=1\cdot 10^{-6}$) sample level differential privacy on Income-A.}
    \label{tab:dp_32}
\end{table}

Table~\ref{tab:dp_32} presents the results with both active and honest-but-curious adversaries. 
Our passive attacks significantly outperform Grad baselines on Income-A and Medical datasets, achieving over 30 and 8 percentage point improvements on Income-A and Medical datasets, respectively. Furthermore, our attacks improve also Grad-w-O baseline by 4 percentage points on Income-A and 3 percentage points on Medical. Surprisingly, passive gradient-based attacks demonstrate better performance than Ours on Income-L, despite having near-zero cosine similarity. This observation suggest that cosine similarity may not be a suitable metric for gradient-based attacks' optimization. Moving to the active adversary, our attacks exhibit minimal performance changes between rounds 10 and 50 on Income-A and Medical datasets, as they approach the empirical optimal performance achievable by Model-w-O. Conversely, on Income-L, where performance margins are larger, our attack surpass Grad by 3 and 12 percentage points after 10 and 50 active rounds, respectively, with potential for further optimization to match the optimal local model performance.